\documentclass{article}

\usepackage[preprint]{neurips_2026}

\usepackage[utf8]{inputenc} 
\usepackage[T1]{fontenc}    
\usepackage{xcolor}
\usepackage{hyperref}
\hypersetup{
    colorlinks=true,   
    linkcolor=red,     
    citecolor=green,     
    urlcolor=blue      
}
\usepackage{url}            
\usepackage{booktabs}       
\usepackage{amsfonts}       
\usepackage{nicefrac}       
\usepackage{microtype}      
\usepackage{array} 
\usepackage{graphicx}
\usepackage{amsmath}
\usepackage{cleveref}
\usepackage{tabularx}
\usepackage{multirow}
\usepackage{pifont}     
\usepackage{amssymb}
\usepackage{amsmath, amsthm}

\newtheorem{theorem}{Theorem}[section]
\newtheorem{proposition}{Proposition}[section] 
\title{FlowCoMotion: Text-to-Motion Generation via Token–Latent Flow Modeling}

\author{
  Dawei Guan\textsuperscript{1} \quad
  Di Yang\textsuperscript{1, 2} \quad
  Chengjie Jin\textsuperscript{1} \quad
  Jiangtao Wang\textsuperscript{1, 2} \\
  \textsuperscript{1} School of Artificial Intelligence and Data Science, University of Science and Technology of China\\
  \textsuperscript{2} Suzhou Institute for Advanced Research, University of Science and Technology of China\\
  \texttt{guand3841@gmail.com} \\
  \texttt{\{di.yang, cjsyp, wangjiangtao\}@ustc.edu.cn} \\
}

\begin{document}
\newcommand{\theHalgorithm}{\arabic{algorithm}}
\newcommand{\et}[2]{\ensuremath{#1^{\pm #2}}}
\newcommand{\etb}[2]{\ensuremath{\mathbf{#1}^{\pm #2}}}
\newcommand{\ets}[2]{\underline{\ensuremath{#1^{\pm #2}}}}

\maketitle

\begin{abstract}
Text-to-motion generation is driven by learning motion representations for semantic alignment with language. Existing methods rely on either continuous or discrete motion representations. However, continuous representations entangle semantics with dynamics, while discrete representations lose fine-grained motion details.
In this context, we propose \textbf{FlowCoMotion}, a novel motion generation framework that unifies both treatments from a modeling perspective. Specifically, FlowCoMotion employs token-latent coupling to capture both semantic content and high-fidelity motion details. In the latent branch, we apply multi-view distillation to regularize the continuous latent space, while in the token branch we use discrete temporal resolution quantization to extract high-level semantic cues. The motion latent is then obtained by combining the representations from the two branches through a token-latent coupling network.
Subsequently, a velocity field is predicted based on the textual conditions. An ODE solver integrates this velocity field from a simple prior, thereby guiding the sample to the potential state of the target motion. Extensive experiments show that FlowCoMotion achieves competitive performance on text-to-motion benchmarks, including HumanML3D and SnapMoGen.
\end{abstract}

\section{Introduction}
Text-to-motion generation, which diverse, temporally coherent human motions from natural language descriptions, has experienced notable progress in computer vision, graphics, and robotics. These advances have been enabled by large-scale text–motion datasets \cite{HumanML3D,Snapmogen,MotionMillion} and a variety of deep generative models, including VAEs \cite{HumanML3D,TEMOS}, autoregressive models \cite{ScaMo,CoMo,T2M-GPT,MotionGPT}, diffusion models \cite{mdm,MotionCraft,ReMoDiffuse,bad,MotionGPT3,StableMoFusion}, and mask generative models \cite{Momask,Snapmogen}. However, existing models exhibit substantial limitations when handling complex cues. They often struggle to achieve \textbf{fine-grained control} over motion extent and intensity (e.g.,``\textit{clockwise half circle}'' or ``\textit{ninety degrees}''), and to capture subtle variations in human motion. Autoregressive models use vector quantization, which introduces approximation errors, thus limiting the quality of motion generation. In autoregressive generation, tokens at each step is generated conditionally on previously generated context, and accumulate errors over time. While diffusion models can mitigate approximation error, it typically requires much denoising steps while inferring, making long-horizon motion synthesis computationally expensive and latency-intensive. Moreover, repeatedly refining high-dimensional trajectories can lead to temporal drift, so achieving high-fidelity often incurs substantial sampling cost.

\begin{figure*}[thbp]
    \centering
    \includegraphics[width=\linewidth]{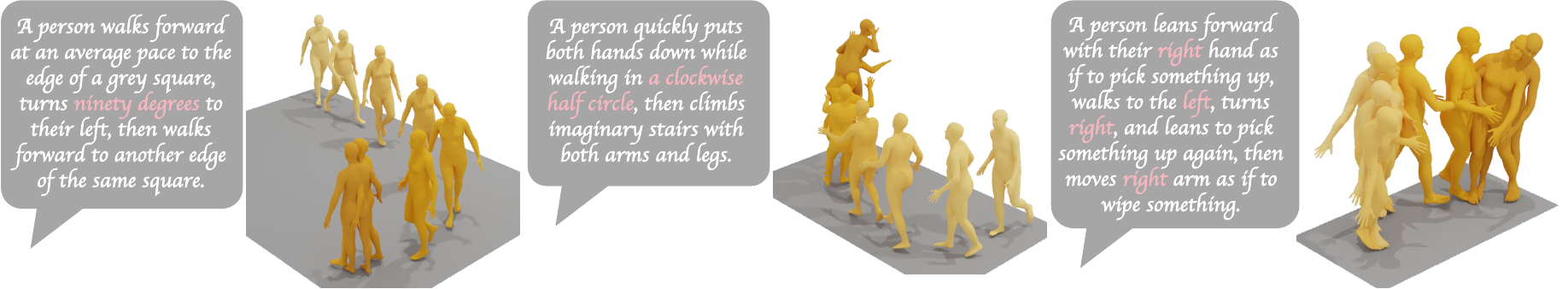}
    \caption{\footnotesize We propose a novel framework called \textbf{FlowCoMotion}, which can handle fine-grained semantic content to guide the generation of high-quality motion. This method is particularly suitable for handling various forms of input, including but not limited to numbers, shapes and directions.}
    \label{fig:teaser}
\end{figure*}

Based on these observations, we propose a novel framework called \textbf{FlowCoMotion} which can perform high-quality text-motion generation and understand precise text. FlowCoMotion combines discrete vectors obtained through residual vector quantization (RVQ) with continuous vectors and performs conditional generation via the flow matching strategy. Our framework is based on two components. First, we introduce a VAE with two branches, Token and Latent, to establish an accurate mapping between 3D motion and its corresponding token-latent coupling vector. Unlike previous methods that typically use hierarchical RVQ alone for iterative residual quantization or latent vector representation, our VAE design combines two distinct representations: the discrete tokens and the continuous latent, which are  generated from the token and latent branch, respectively. Second, the Transformer-based generation head parameterizes a text-conditional velocity field that transfers samples from a simple prior distribution to the target motion latent by flow matching. While diffusion methods also start with noise, they typically require multiple denoising iterations during the inference phase, increasing computational cost and potentially leading to long-term time drift accumulation. Therefore, an ODE solver integrates the predicted velocity field, mapping the prior to the conditional motion latent. Since the model predicts sequence-level dynamics rather than subsequent autoregressive updates, it leverages a full-context attention mechanism to better maintain long-range consistency while avoiding the multi-step sampling of diffusion methods.

Our main contributions can be summarized into three parts: 
(a). We propose FlowCoMotion, a framework for text‑motion generation that improves fine‑grained text alignment. Our method shows enhanced performance on text descriptions containing numeric values, directional cues, and geometric shapes.
(b). We introduce the token-latent coupling method to simultaneously guarantee the semantic content and dynamic features of motion representation. Besides, we also introduce a flow matching generation strategy for accurate and efficient generation. (c). Experimental results demonstrate that our FlowCoMotion achieves competitive performance in text-to-motion generation tasks. On the HumanML3D dataset, our R-precision values reach $0.55$, $0.74$, and $0.83$, respectively.

\section{Related Work}
\textbf{Text-to-Motion Generation.}
Recently, text-motion generation has made significant progress, improving the realism of motion and its alignment with text description. Early methods used variational autoencoders to model motion \cite{HumanML3D,TEMOS}. Vector quantization represents motion into discrete tokens, which can be generated using autoregressive methods \cite{ScaMo,CoMo,T2M-GPT,MotionGPT,GenM3} or mask-based methods \cite{Momask,Snapmogen}. Meanwhile, the diffusion methods \cite{mdm,MotionDiffuse,MLD,ReMoDiffuse} denoise the motion samples iteratively, making them better match the text description. Furthermore, some researchers have achieved more advanced semantic alignment capabilities using autoregressive diffusion models \cite{MotionGPT3}. Despite these advances, most work still relies on continuous or discrete representations, making it difficult to simultaneously utilize labeled semantic structure and continuous temporal dynamics. It is precisely because of this gap that we propose a labeled-latent coupling method.

\textbf{Hybrid Representations in Generative Modeling.} High-fidelity visual synthesis increasingly leverages the synergy between discrete and continuous representations. In image generation, VQ-GAN \cite{VQ_GAN} established a powerful hybrid paradigm: using codebooks to compress structural priors while performing continuous optimization in latent space. Similarly, video models like MAGVIT \cite{MAGVIT} and VideoLDM \cite{VideoLDM} employ 3D tokens for temporal consistency alongside continuous transitions for inter-frame smoothness. By bridging discrete semantic synthesis and continuous texture synthesis, these frameworks alleviate high-dimensional complexity while preserving fine perceptual details. This rationale provides the foundation for our token-latent coupling method. However, motion generation poses unique challenges not present in image or video generation: (1) Kinematics demand plausible joint rotations, (2) temporal dynamics must preserve continuity and causality, and (3) semantics entangle with trajectories. Unlike previous hybrid methods, our approach preserves continuous dynamics under kinematic regularization.


\begin{figure*}[thbp]
    \centering
    \includegraphics[width=\linewidth]{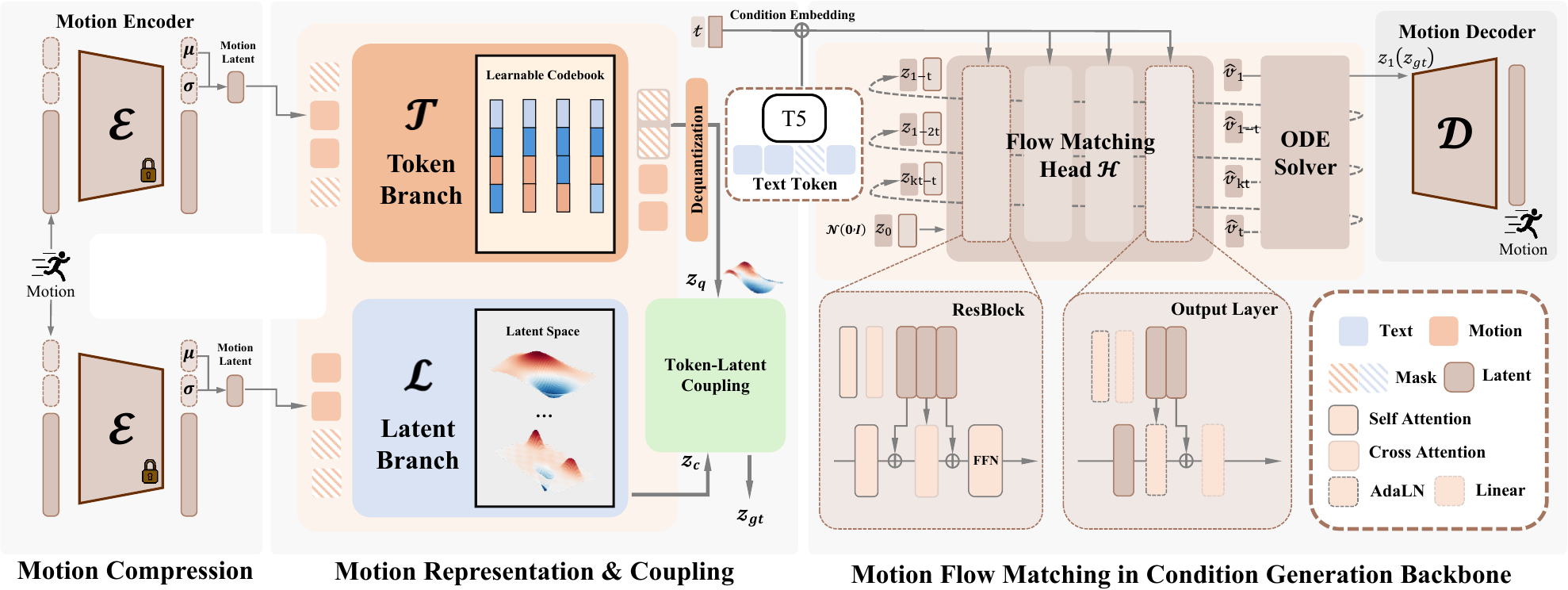}
    \caption{\footnotesize \textbf{Method Overview.} FlowCoMotion mainly includes three key components: (1) Motion compression (\Cref{section 3.1}): The motion sequence is compressed into the latent space using VAE to obtain $z \in \mathbb R^{n\times b}$. (2) Motion representation \& Coupling (\Cref{section 3.2,section 3.3}): The Token Branch further uses RVQ-VAE to extract discrete Token representations and then dequantizes them to obtain $z_{q}$. The Latent Branch directly uses the continuous latent space vector $z_{c}$. Finally, the final motion representation $z_{gt}$ is obtained through the token-latent coupling network. (3) Motion Flow Matching in Condition Generation (\Cref{section 3.4}): The flow velocity field $\hat v_{t}$ is predicted through the text condition $c$ and time $t$. The predicted latent vector $\hat z_{1}$ is solved by the ODE solver and finally handed over to the Decoder $\mathcal D$.}
    \label{fig:overview}
\end{figure*}

\section{Methodology}
We propose \textbf{FlowCoMotion}, a framework for efficient text-to-motion generation. 
As shown in \Cref{fig:overview}, FlowCoMotion consists of three components: (1) Motion Compression (\Cref{section 3.1}): A shared VAE encoder $\mathcal E$ maps motion sequences to a compact latent space $z$. (2) Motion Representation \& Coupling (\Cref{section 3.2,section 3.3}): The latent $z$ is processed by two parallel branches, a Token Branch that applies RVQ quantization to obtain discrete semantic tokens $z_q$, and a Latent Branch that preserves continuous dynamics $z_c$. A coupling network $\mathcal C$ fuses $(z_q, z_c)$ into a unified representation $z_{gt}$. (3) Motion Flow Matching (\Cref{section 3.4}): A text-conditional flow matching head $\mathcal H$ predicts the velocity field that maps noise to $z_{gt}$.

We use a VAE with encoder $\mathcal E$ and decoder $\mathcal D$ to map an $M$-frame motion sequence $m^{1:M}$ to a compact continuous latent $z\in\mathbb R^{d}$, and reconstruct it via $\mathcal D$. Given a text input with $N$ tokens $w^{1:N}$ which can be mapped into condition $C$, FlowCoMotion samples a noise latent conditioned on $C$ and applies the flow matching head $\mathcal H$ to predict the motion latent $\hat z$. The synthesized motion is then obtained by decoding $\hat m^{1:M}=\mathcal D(\hat z)$.

\subsection{Motion Compression}
\label{section 3.1}
We employ a VAE to compress the motion sequence into a continuous latent space. Our transformer-based VAE \cite{Momask}, comprising an encoder $\mathcal{E}$ and decoder $\mathcal{D}$ with long-stage skip connections, encodes the motion sequence into a compressed, information-rich latent space, preserving fine-grained dynamics and kinematic trends, and supporting high-fidelity, diverse motion generation.

For an input motion sequence $x^{1:L}$ of arbitrary length $L$, the encoder processes the motion vector and a set of learnable distribution labels to infer a Gaussian distribution with mean $\mu$ and standard deviation $\sigma$.

\subsection{Motion Representation in Discrete Token}
\label{section 3.2}
Accurately representing motion sequences is essential for motion generation tasks. We propose a \textbf{dual-representation} framework that combines token and latent representations, encoding motion sequences through two branches: one for discrete tokens and the other for continuous latent space to capture temporal details. Discretization inherently leads to information loss, reducing expressive power during quantization. To mitigate this, our framework integrates continuous latent space and discrete representations, enabling the simultaneous capture of subtle temporal variations and semantic content in motions.

For the \textbf{Token Branch}, we treat the token distribution as semantic anchors, performing quantization and dequantization. To represent motion semantics with discrete tokens, we train a motion quantizer $\mathcal{T}$ based on the VQ-VAE architecture, as applied in \cite{T2M-GPT,MotionGPT,Momask,MotionStreamer}. These approaches enable efficient condensation of motion semantics, facilitating fusion with continuous temporal details for motion generation tasks.

To enable the Token Branch as a semantic anchor, we employ an RVQ-VAE quantizer architecture. Multiple quantization layers are introduced to perform quantization at different temporal resolutions, with all layers sharing a single codebook \cite{Snapmogen}. For a motion latent feature sequence $z \in \mathbb{R}^{n \times d}$, the quantizer applies a series of residual quantization operations with progressively increasing temporal resolutions $\{h^v\}_{v=0}^V$, where $h^0 < \dots < h^V = n$. The quantized features $z_q \in \mathbb{R}^{h_0 \times d}$ at the coarse-grained level are restored to full resolution via bilinear interpolation: $\hat{f}_{\uparrow}^v = \mathcal{I}(f^v, h^v)$. Residuals are computed and downsampled to the next scale for quantization by subsequent layers.
\begin{equation}
\hat{z}_{q}=\sum_{v=0}^V\mathcal{I}(\hat{f}^v,h^v)
\end{equation}
Overall, this Quantizer trained using latent embedding loss (commit loss) on each quantization scale $h\in\{h_{0},\ldots,h_{V}\}$:
\begin{equation}
\mathcal{L}_{\mathrm{commit}}^{(h)}=\frac{1}{|\mathcal{M}_h|}\sum_{m\in\mathcal{M}_h}\left\|z_e^{(m,h)}-\mathbf{e}_{k^{(m,h)}}\right\|_2^2
\end{equation}
Where $z_e^{(m,h)}$ represents the continuous latent variable of the encoder at scale $h$ for motion sample $m$, and $\mathbf{e}_{k^{(m,h)}}$ is the selected vector from the codebook, determined by nearest neighbor quantization. $\mathcal{M}_h$ denotes the set of valid locations at scale $h$, and $|\mathcal{M}_h|$ is the number of its elements. By employing temporal resolution sampling, our Token Branch effectively samples semantic anchors at different time intervals, optimizing its fidelity and versatility in motion generation tasks.

\subsection{Motion Representation in Continuous Latent}
\label{section 3.3}
In the \textbf{Latent Branch}, a latent vector $z \in \mathbb{R}^{m}$ is extracted from the distribution using a reparameterization technique. This latent vector is fused with the Token Branch output and fed into the decoder $\mathcal{D}$ to generate the reconstructed motion sequence. The reconstruction training process is detailed in \Cref{section 3.5}.

\textbf{Multi-view Distillation.} Relying solely on reconstruction and quantization constraints may lead to latent space inconsistencies under varying pruning/enhancement conditions, reducing generalization for downstream conditional generation. To address this, we propose a multi-view teacher-student self-distillation framework \cite{self-supervised,self-distill}, constructing a multi-view set $v$ for the same sequence. The teacher processes the global view set $v_g$, while the student processes $v$. After passing through projection heads $g_{t}$ and $g_{s}$, and decentralization, we obtain the probability distribution.
\begin{equation}
        \mathbf{p}_t(v) = \text{Softmax}\left(\frac{\scriptstyle g_t(\mathcal E_t(v)) - \mathbf{c}}{\tau_t}\right) \quad
        \mathbf{p}_s(v) = \text{Softmax}\left(\frac{\scriptstyle g_s\left(\phi(\mathcal E_s(v))\right)}{\tau_s}\right)
\end{equation}

Where $\phi(\cdot)$ denotes the polymerization operator for distillation, and $\mathbf{c}$ represents the center vector, updated using exponential moving average to prevent characterization collapse. The distillation loss \cite{Distill} is computed on a global view using the KL divergence between two probability distributions. The final latent sequence approximation $\hat z_{q}$ is the sum of all up-interpolated quantized sequences, which is passed through a \textbf{token-latent coupling network} $\mathcal{C}$ with the latent branch prediction $\hat{z}_c$, and subsequently fed into the Decoder $\mathcal{D}$ to predict the motion sequence.
\begin{equation}
\hat{\mathbf{m}}=\mathcal{D}(\mathcal C(\hat{z}_{q}, \hat z_{c}))
\end{equation}
Through multi-view distillation, our Latent Branch is effectively characterized, enhancing its capacity to capture motion details and optimize motion generation performance. In Appendix \Cref{sec:theory_coupling}, we provide a concrete proof of the theoretical feasibility of token-latent coupling.

\subsection{Motion Flow Modeling in Condition Generation}
\label{section 3.4}

To improve fidelity and diversity in motion generation, we propose a text-conditioned flow matching strategy that maps random noise to a token-latent coupled latent space. Building on recent advances in flow matching generative models, our approach integrates Rectified Flow \cite{FlowMatching,RectifiedFlow} with a conditional Transformer backbone. The flow matching head predicts the velocity field in the motion latent space with high accuracy, introducing minimal inference overhead.

\begin{figure*}[htbp]
    \centering
    \includegraphics[width= 0.5\linewidth]{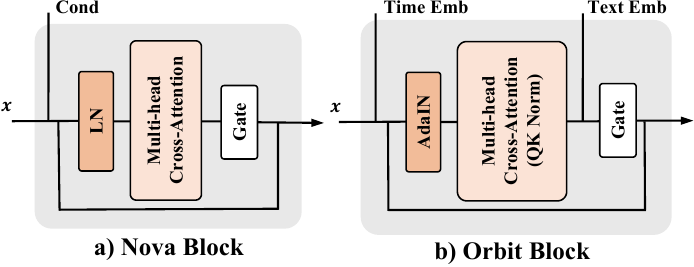}
    \caption{\footnotesize \textbf{Two Conditional Transformer Blocks for Motion Generation.} (a) The \textbf{Nova} Block integrates a unified conditioning signal via LN-modulated cross-attention and uses a learnable gate to control the strength of the conditioned update on the hidden state $x$. (b) The \textbf{Orbit} Block separates conditioning into time and text inputs, performing cross-attention with $QK$ normalization for improved stability. It also applies the same learnable gate to modulate the conditioned residual.}
    \label{fig:transformer}
\end{figure*}

\textbf{Flow Matching Process.} We apply flow matching principles to represent complex probability distributions within the motion latent space. Instead of adding Gaussian noise over discrete time steps as in diffusion, we construct a continuous probability path connecting a simple prior to the real data latent variables over time $t \in [0,1]$, resulting in a sequence of intermediate states ${z_t}$. Given a real latent variable $z_1$ from encoding real motions and a random sample $z_0 \sim \mathcal{N}(0,I)$ from a standard Gaussian distribution, we generate intermediate samples via a simple linear path.
\begin{equation}
z_t=tz_1+(1-t)z_0
\end{equation}
$t$ controls the interpolation position. Unlike Diffusion process, flow matching requires specifying a target instantaneous velocity for each $z_t$, which can be viewed as a shortest straight line in the path.
\begin{equation}
u_{t}=\frac{dz_{t}}{dt}=z_{1}-z_{0}
\end{equation}
The conditional flow model $v_\theta(z_t,t,c)$, where $c$ is the auxiliary input, is trained to predict the instantaneous velocity $u_t$, learning the transformation from the prior to the data distribution over continuous time. In our setup, $c$ corresponds to the output hidden state of the Transformer backbone, providing conditional information that modulates the vector field. We use \texttt{T5-base} \cite{T5} to extract word-level features from complex textual descriptions. As shown in \Cref{fig:transformer}, we present two types of Transformers for conditional generation, each injecting distinct conditions. More details for our model are provided in Appendix \Cref{appendix_2models}. The  objective is based on the flow matching loss.
\begin{equation}
\mathcal{L}=\mathbb{E}_{z_0,z_1,t}\left[\left\|(z_{1}-z_{0})-v_\theta(z_t,t,c)\right\|_2^2\right]
\end{equation}
where $z_{1}$ is sampled from the token-latent coupling network using the ground truth motion latent.

\textbf{Inference Procedure.} We begin by sampling initial latent variables $z_0 \sim \mathcal{N}(0,I)$ from random noise and providing a conditional input $c$. Next, we define a continuous time generative dynamical system using the learned conditional velocity field $v_\theta(z,t,c)$. We obtain the generated latent variables $\hat z_1$ by solving ordinary differential equations (ODEs) to integrate the state from $t=0$ to $t=1$. Finally, $\hat z_1$ is input into the motion decoder $\mathcal D$ to reconstruct the final motion sequence. To ensure efficient conditional generation, we use a fixed-step \textbf{ODE solver} with small sampling steps to map the data distribution.

\begin{figure*}[t]
    \centering
    \includegraphics[width=\linewidth]{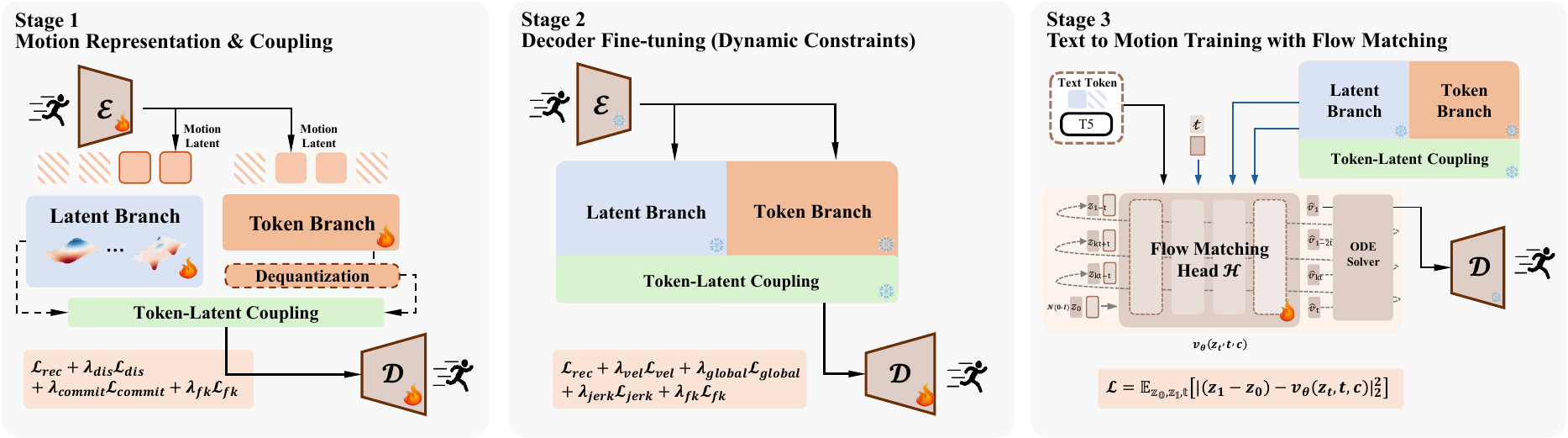}
    \caption{\footnotesize We propose a \textbf{three-stage training} strategy for the motion-language model: \textbf{Stage 1:} Train the Encoder $\mathcal E$, Decoder $\mathcal D$, and learnable latent space, refining the latent space representation. \textbf{Stage 2:} Freeze the Encoder and learnable latent space, and fine-tune the Decoder by kinematic losses. \textbf{Stage 3:} Train the flow matching head $\mathcal H$ to predict the velocity flow from random noise to the motion latent space. \textbf{Inference process:} Conditions are injected at each block, and the predicted latent vector is obtained by solving the ODE through $N$ steps after sampling from random noise. Finally, it is fed into the Decoder to obtain the predicted motion sequence.}
    \label{fig:trainstage}
\end{figure*}

\subsection{Training Procedure Overview}
\label{section 3.5}
Fully end-to-end training can destabilize the latent space and weaken kinematic constraints. We therefore adopt a three-stage training pipeline that progressively learns the motion representation, stabilizes kinematic reconstruction, and enables diverse generation. As shown in \Cref{fig:trainstage}, \textbf{Motion Representation \& Coupling} learns motion sequence representations that support both reconstruction and generation. \textbf{Decoder Fine-tuning with Dynamic Constraints} freezes the learned motion representation and fine-tunes the Decoder to better satisfy kinematic requirements during reconstruction. \textbf{Text-to-Motion Training with Flow Matching} trains the generation head to predict the velocity field that samples into the learned motion representation.

\textbf{Motion Representation \& Coupling.} To map motions into our predefined latent space and achieve accurate reconstruction, we first train the VAE, Token Branch, and Latent Branch. This stage prioritizes enriching the latent representation while minimizing the impact of kinematic constraints. Using the multi-view distillation, we encourage a richer mapping from motion sequences to the latent space. Following \cite{T2M-GPT,fkloss}, we optimize the motion tokenizer with four losses: $\mathcal L= \mathcal L_{rec}+\lambda_{dis}\mathcal L_{dis}+\lambda_{commit}\mathcal L_{commit}+\lambda_{fk}\mathcal L_{fk}$, where $\mathcal L_{rec}$ is the reconstruction loss, $\mathcal L_{commit}$ is the commitment loss, $\mathcal L_{dis}$ is the multi-view distillation loss, and $\mathcal L_{fk}$ is the forward kinematics (FK) loss.

\textbf{Decoder Fine-tuning with Dynamic Constraints.} In this stage, we freeze the encoder, token \& latent branches, and introduce kinematic constraints to fine-tune the decoder. Specifically, we apply a series of kinematic constraints, including jerk loss and global loss. This stage preserves the effective motion sequence representation from the first stage, enabling the decoder to better reconstruct the motion sequence. In summary, this stage performs rigorous kinematic fine-tuning on the decoder, with a loss of $\mathcal L= \mathcal L_{rec}+\lambda_{vel}\mathcal L_{vel}+\lambda_{global}\mathcal L_{global}+\lambda_{jerk}\mathcal L_{jerk}+\lambda_{fk}\mathcal L_{fk}$, where the jerk loss $\mathcal L_{jerk}$, the global loss $\mathcal L_{global}$, the velocity loss $\mathcal L_{vel}$.

\textbf{Text-to-Motion Training with Flow Matching.} In the final stage, we encode the text caption with \texttt{T5-base} and feed it to the flow matching head. We freeze all networks used for motion reconstruction. Conditioned on time and text, the head models the sampled path by predicting the velocity field. We train it with an \texttt{L2} loss between the predicted and ground truth velocity. This formulation leverages the sequence reconstruction capability learned in earlier stages while preserving inference efficiency.

\begin{figure*}[ht]
    \centering
    \includegraphics[width=\linewidth]{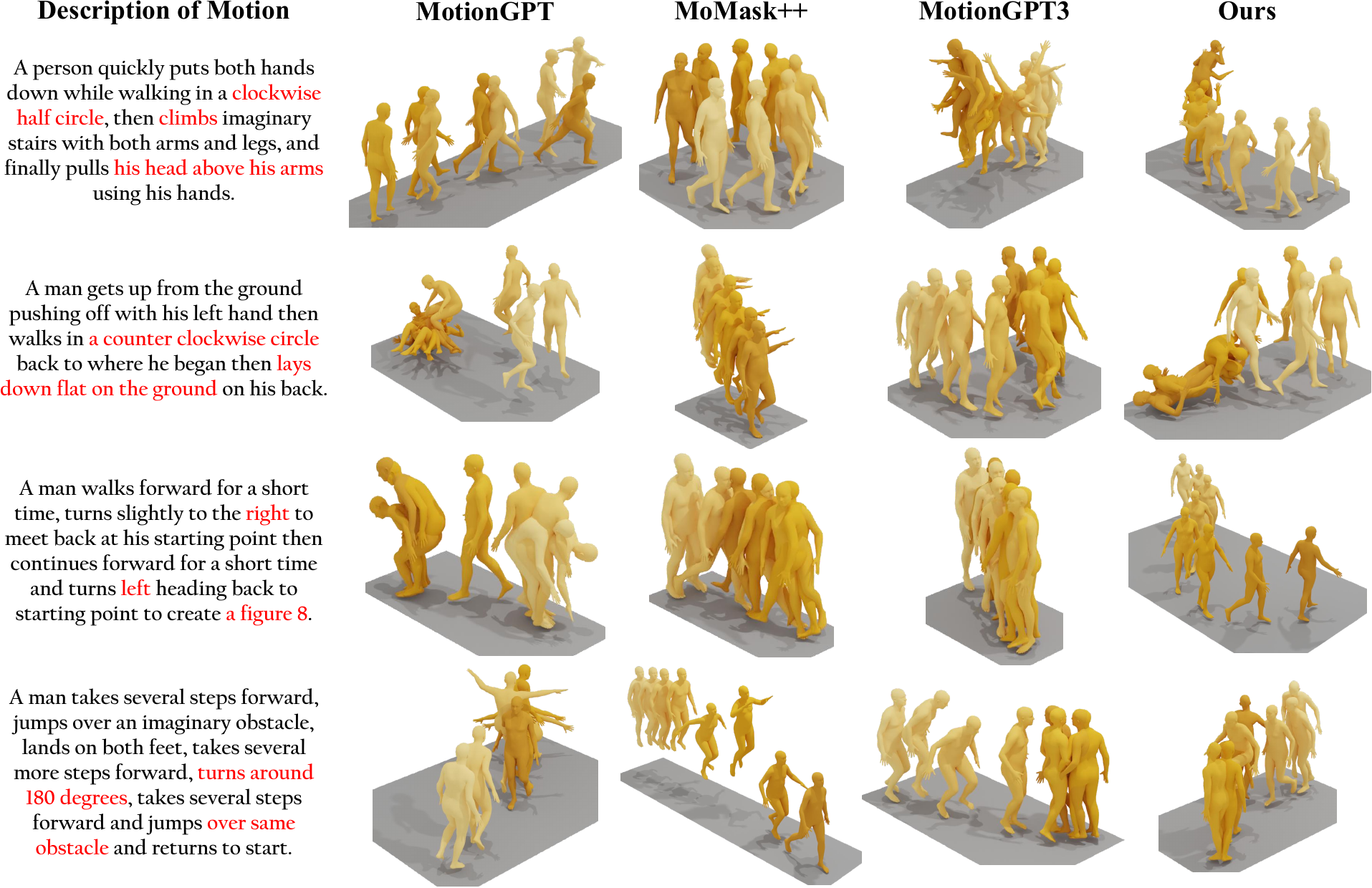}
    \caption{\footnotesize \textbf{Qualitative evaluation on text-to-motion generation.} We qualitatively compared the visualizations generated by our method with those produced by MotionGPT \cite{MotionGPT}, MoMask++ \cite{Snapmogen}, and MotionGPT3 \cite{MotionGPT3}. Misaligned motions of other methods are marked with red text. The results show that \textbf{FlowCoMotion} demonstrates powerful text understanding capabilities, enabling it to generate more accurate and coherent motions.}
    \label{fig:experiment1}
\end{figure*}

\section{Experiment}
Extensive experiments demonstrate the superior performance of our model, \textbf{FlowCoMotion}, in motion generation tasks. Details of the dataset settings, evaluation metrics, and implementation specifics are provided in \Cref{section 4.1}. We first present the performance of text-to-motion generation in \Cref{section 4.2}, through both quantitative and qualitative experiments. Next, in \Cref{section 4.3}, we will present more ablation experiments to demonstrate the feasibility of our method. The supplements include additional results and further implementation details.

\subsection{Experimental Setup}
\label{section 4.1}
\textbf{Datasets.} Our experiments are based on the HumanML3D and SnapMoGen \cite{HumanML3D,Snapmogen} datasets. HumanML3D is a large-scale benchmark for text-to-motion generation and understanding, containing 14,616 motion sequences from AMASS \cite{AMASS} and HumanAct12, annotated with 44,970 sequence-level natural language descriptions. SnapMoGen includes 43.7 hours of high-quality motion data at 30 frames per second, with each motion segment accompanied by 6 detailed text descriptions, including daily tasks, fitness, social interactions, and dance.

\textbf{Evaluation Metrics.} We evaluate motion quality using Fréchet Inception Distance (FID), which measures how closely generated motions align with ground truth motions in feature space, reflecting overall motion quality. Diversity quantifies feature space variation across generated samples, while MultiModality captures variation among motions generated from the same text description. To assess text-motion alignment, we use motion-text retrieval R-Precision and Multimodal Distance, which measures the embedding-space distance between paired motion and text embeddings.

\textbf{Implementation Details.} FlowCoMotion consists of a motion VAE and a Transformer-based flow matching head. The VAE has 6 layers and encodes each sequence into a 256 dimension latent, split into a 64 dimension token branch and a 192 dimension continuous branch. The flow matching head uses 15 Transformer layers with 8 attention heads and a 512 hidden size. We perform inference with Rectified Flow sampling using 40 euler steps. The VAE is trained for 400 epochs and the flow head for 1000 epochs, on a single NVIDIA RTX 4090 GPU. More details are reported in Appendix \Cref{More_details}.

\begin{table*}[htbp]
    \caption{\footnotesize\textbf{Quantitative evaluation on \text{HumanML3D} test set.}}
    \centering
    \scalebox{0.75}{
    \begin{tabular}{l c c c c c c c}
    \toprule
     \multirow{2}{*}{Methods}  & \multicolumn{3}{c}{R-precision$\uparrow$} & \multirow{2}{*}{FID$\downarrow$} & \multirow{2}{*}{MultiModal Dist$\downarrow$} & \multirow{2}{*}{Diversity$\rightarrow$} & \multirow{2}{*}{MModality$\uparrow$}\\
    \cline{2-4}
     ~ & Top 1 & Top 2 & Top 3 \\
    \midrule
       Real motions & \et{0.511}{.003} & \et{0.703}{.003} & \et{0.797}{.002} & \et{0.002}{.000} & \et{2.974}{.008} & \et{9.503}{.065} & - \\
    \midrule
    TM2D~ & \et{0.319}{.000} & - & - & \et{1.021}{.000} & \et{4.098}{.000} & \etb{9.513}{.000} & \etb{4.139}{.000} \\
    MotionCraft~ & \et{0.501}{.003} & \et{0.697}{.003} & \et{0.796}{.002} & \et{0.173}{.002} & \et{3.025}{.008} & \et{9.543}{.098} & - \\
    ReMoDiffuse~ & \et{0.510}{.005} & \et{0.698}{.006} & \et{0.795}{.004} & \et{0.103}{.004} & \et{2.974}{.016} & \et{9.018}{.075} & \et{1.795}{.043} \\
    MMM~ & \et{0.504}{.003} & \et{0.696}{.003} & \et{0.794}{.002} & \et{0.080}{.003} & \et{2.998}{.007} & \et{9.411}{.058} & \et{1.164}{.041} \\
    MotionGPT~ & \et{0.492}{.003} & \et{0.681}{.003} & \et{0.778}{.002} & \et{0.232}{.008} & \et{3.096}{.008} & \et{9.528}{.071} & \et{2.008}{.084} \\
    MotionGPT-2~ & \et{0.496}{.002} & \et{0.691}{.003} & \et{0.782}{.004} & \et{0.191}{.004} & \et{3.080}{.013} & \et{9.860}{.026} & \et{2.137}{.022} \\
    DiverseMotion~ & \et{0.515}{.003}  & \et{0.706}{.002} & \et{0.802}{.002} & \et{0.072}{.004} & \et{2.941}{.007} & \et{9.683}{.102} & \et{1.869}{.089} \\
    BAD~ & \et{0.517}{.002} & \et{0.713}{.003} & \et{0.808}{.003} & \et{0.065}{.003} & \et{2.901}{.008} & \et{9.694}{.068} & \et{1.194}{.044} \\  
    MoMask~ & \et{0.521}{.002} & \et{0.713}{.002} & \et{0.807}{.002} & \etb{0.045}{.002} & \et{2.958}{.008} & \et{9.620}{.064} & \et{1.241}{.040} \\
    MoMask++~ & \et{0.528}{.003} & \et{0.718}{.003} & \et{0.811}{.002} & \et{0.072}{.003} & \et{2.912}{.008} & \et{9.573}{.062} & \et{1.227}{.046} \\
    MotionGPT3~ & \ets{0.543}{.003} & \ets{0.735}{.002} & \ets{0.828}{.002} & \et{0.217}{.010} & \etb{2.793}{.007} & \et{9.662}{.072} & \et{1.366}{.046} \\
    \midrule
    Ours$^\text{nova}$  & \et{0.532}{.002} & \et{0.725}{.002} & \et{0.819}{.002} & \et{0.072}{.003} & \et{2.887}{.006} & \et{9.377}{.078} & \et{1.875}{.088} \\
    Ours$^\text{orbit}$  & \etb{0.550}{.003} & \etb{0.740}{.003} & \etb{0.832}{.002} & \ets{0.061}{.003} & \ets{2.852}{.008} & \ets{9.457}{.067} & \ets{2.071}{.093} \\
    \bottomrule
    \end{tabular}
    }
    \label{table_hml}
\end{table*}

\subsection{Comparison to State-of-the-art Approaches}
\label{section 4.2}
We compared our method with a range of existing state-of-the-art works, including autoregressive models \cite{TM2D,MMM,MotionGPT,motiongpt2,T2M-GPT,MARDM}, diffusion models \cite{mdm,MotionCraft,ReMoDiffuse,bad,StableMoFusion}, mask generation models \cite{Momask,Snapmogen}, and autoregressive diffusion model \cite{MotionGPT3}.

\begin{table*}[t]
    \caption{\footnotesize\textbf{Quantitative evaluation on \text{SnapMoGen} test set.}}
    \centering
    \scalebox{0.85}{
    \begin{tabular}{l c c c c c c}
    \toprule
     \multirow{2}{*}{Methods}  & \multicolumn{3}{c}{R-precision$\uparrow$} & \multirow{2}{*}{FID$\downarrow$} & \multirow{2}{*}{CLIP Score$\uparrow$} & \multirow{2}{*}{MModality$\uparrow$}\\
    \cline{2-4}
     ~ & Top 1 & Top 2 & Top 3 \\
    \midrule
       Real motions & \et{0.940}{.001} & \et{0.976}{.001} & \et{0.985}{.001} & \et{0.001}{.000} & \et{0.837}{.000} & - \\
    \midrule
     MDM~ & \et{0.503}{.002} & \et{0.653}{.002} & \et{0.727}{.002} & \et{57.783}{.092} & \et{0.481}{.001} & \etb{13.412}{.231} \\
     T2M-GPT~ & \et{0.618}{.002} & \et{0.773}{.002} & \et{0.812}{.002} & \et{32.629}{.087} & \et{0.573}{.001} & \et{9.172}{.181} \\
     StableMoFusion~ & \et{0.679}{.002} & \et{0.823}{.002} & \et{0.888}{.002} & \et{27.801}{.063} & \et{0.605}{.001} & \et{9.064}{.138}  \\
     MARDM~ & \et{0.659}{.002} & \et{0.812}{.002} & \et{0.860}{.002} & \et{26.878}{.131} & \et{0.602}{.001} & \et{9.812}{.287}  \\
     MoMask~ & \ets{0.777}{.002} & \et{0.888}{.002} & \et{0.927}{.002} & \et{17.404}{.051} & \et{0.664}{.001} & \et{8.183}{.184}  \\
    MoMask++ & \etb{0.805}{.002} & \etb{0.904}{.002} & \etb{0.938}{.001} & \ets{15.560}{.071} & \etb{0.684}{.001}  & \et{6.556}{.178}  \\
    \midrule
     Ours$^\text{nova}$  & \et{0.737}{.002} & \et{0.873}{.002} & \et{0.929}{.002} & \et{17.544}{.068} & \et{0.654}{.001} & \ets{9.190}{.240} \\
     Ours$^\text{orbit}$  & \et{0.776}{.002} & \ets{0.890}{.002} & \ets{0.934}{.001} & \etb{14.678}{.061} & \ets{0.672}{.001} & \et{8.553}{.230} \\
    \bottomrule
    \end{tabular}
    }
    \label{table_snap}
\end{table*}


\textbf{Qualitative Comparison.} \Cref{fig:experiment1} presents a qualitative comparison between our method, MotionGPT, MoMask++, and MotionGPT3. While MotionGPT generates semantically correct motions, it struggles with direction- and degree-related descriptions, such as ``clockwise half circle'' and ``in figure 8''. Although MoMask++ and MotionGPT3 show improvements, they still fail to fully align motions with the text descriptions. Additionally, MoMask++ occasionally produces stiff and repetitive motions. In contrast, our method generates high-quality motions that are highly consistent with the input text. Further comparative results can be found in Appendix \Cref{more_results}.

\textbf{Quantitative Comparisons.} Following prior work, we repeat each experiment 20 times and report the mean with a $95\%$ confidence interval. We evaluate two conditional Transformer blocks, termed \textbf{Nova} and \textbf{Orbit}. Results on HumanML3D and SnapMoGen are reported in \Cref{table_hml,table_snap}. FlowCoMotion achieves competitive or better performance on the core metrics, indicating robust text-to-motion generation. On HumanML3D, the \textbf{Nova} transformer already matches strong baselines, and \textbf{Orbit} further improves performance. Besides, we report a full comparison with prior T2M approaches in Appendix \Cref{more_evaluation}. 

\begin{table*}[htbp]
    \caption{\footnotesize \textbf{Ablation analysis of motion representation on Text-to-Motion Tasks.}}
    \centering
    \scalebox{0.75}{
    \begin{tabular}{c c c c c c c c c}
    \toprule
    & \multicolumn{4}{c}{VQ Config.} & \multicolumn{2}{c}{VQ Reconstruction} & \multicolumn{2}{c}{T2M Generation} \\
    \cline{2-5} \cline{6-7} \cline{8-9}
    & \footnotesize\#Codes & \footnotesize\#Quant. & \footnotesize F/M & \footnotesize w/ Att. 
    & \footnotesize FID $\downarrow$ & \footnotesize Joint Pos. Err. $\downarrow$ 
    & \footnotesize FID $\downarrow$ & \footnotesize R-Precison@3 $\uparrow$ \\
    \cmidrule{2-9}
    \multirow{7}{*}[4.0ex]{\rotatebox[origin=c]{90}{HumanML3D}} & 1024 & 4 & M & \checkmark & 0.0622 & 0.0488 & \et{0.122}{.002} & \et{0.809}{.003} \\ 
    & 512 & 4 & M & \checkmark & 0.0859 & 0.0514 & \et{0.218}{.003} & \et{0.771}{.002} \\
    \cmidrule{2-9}
    & \multicolumn{2}{c}{Representation} & \multicolumn{2}{c}{Dimension} 
    & \multicolumn{2}{c}{VQ Reconstruction} & \multicolumn{2}{c}{T2M Generation} \\
    \cline{2-3} \cline{4-5} \cline{6-7} \cline{8-9}
    & \multicolumn{2}{c}{\footnotesize Type} & \multicolumn{2}{c}{\footnotesize Dim.} 
    & \footnotesize FID $\downarrow$ & \footnotesize Joint Pos. Err. $\downarrow$ 
    & \footnotesize FID $\downarrow$ & \footnotesize R-Precison@3 $\uparrow$ \\
    \cmidrule{2-9}
    & \multicolumn{2}{c}{Latent} & \multicolumn{2}{c}{256} & 0.0054 & 0.0302 & \et{0.087}{.004} & \et{0.796}{.003} \\  
    & \multicolumn{2}{c}{Hybrid} & \multicolumn{2}{c}{64, 192} & 0.0036 & 0.0234 & \et{0.061}{.003} & \et{0.832}{.002} \\ 
    \midrule
    & \multicolumn{4}{c}{VQ Config.} & \multicolumn{2}{c}{VQ Reconstruction} & \multicolumn{2}{c}{T2M Generation} \\
    \cline{2-5} \cline{6-7} \cline{8-9}
    & \footnotesize\#Codes & \footnotesize\#Quant. & \footnotesize F/M & \footnotesize w/ Att. 
    & \footnotesize FID $\downarrow$ & \footnotesize Joint Pos. Err. $\downarrow$ 
    & \footnotesize FID $\downarrow$ & \footnotesize R-Precison@3 $\uparrow$ \\
    \cmidrule{2-9}
    \multirow{8}{*}[4.0ex]{\rotatebox[origin=c]{90}{SnapMoGen}} & 1024 & 2 & M & \checkmark & 4.26 & 8.48 & \et{26.248}{.060} & \et{0.911}{.002} \\
    & 512 & 2 & M & \checkmark & 4.41 & 8.62 & \et{28.013}{.072} & \et{0.903}{.001} \\
    \cmidrule{2-9}
    & \multicolumn{2}{c}{Representation} & \multicolumn{2}{c}{Dimension} 
    & \multicolumn{2}{c}{VQ Reconstruction} & \multicolumn{2}{c}{T2M Generation} \\
    \cline{2-3} \cline{4-5} \cline{6-7} \cline{8-9}
    & \multicolumn{2}{c}{\footnotesize Type} & \multicolumn{2}{c}{\footnotesize Dim.} 
    & \footnotesize FID $\downarrow$ & \footnotesize Joint Pos. Err. $\downarrow$ 
    & \footnotesize FID $\downarrow$ & \footnotesize R-Precison@3 $\uparrow$ \\
    \cmidrule{2-9}
    & \multicolumn{2}{c}{Latent} & \multicolumn{2}{c}{256} & 0.1650 & 4.3650 & \et{19.465}{.037} & \et{0.895}{.002} \\  
    & \multicolumn{2}{c}{Hybrid} & \multicolumn{2}{c}{64, 192} & 0.0503 & 4.4234 & \et{14.678}{.061} & \et{0.934}{.001} \\ 
    \bottomrule
    \end{tabular}
     }
    \label{tab:alb_1}
\end{table*}

\subsection{Ablation Studies}
\label{section 4.3}
Our research proposes a new strategy for integrating discrete motion representation and continuous motion, and using a flow matching head for conditional generation. This strategy consists of two key components: 1) Token-Latent Coupling, 2) Conditional generation using flow matching. More ablation studies are shown in Appendix \Cref{appden:abl}.  

\textbf{Token-Latent Coupling.} To evaluate the utility of Token-Latent Coupling, we compared our model with variants of motion tokens that were discretized using vector quantization alone (\textit{token dimension = 256}), and with variants that used only continuous motion Latents (\textit{latent dimension = 256}). As shown in \Cref{tab:alb_1}, under the same training iterations, the mixed representation significantly outperformed the other variants in both the motion reconstruction and generation tasks. Particularly for the transformation from text to motion, discretization would disrupt the temporal coherence between different motion segments and introduce quantization noise, which would hinder the generation quality. In contrast, the mixed latent vectors retain a compact and expressive semantics, enabling smoother generation, and can be better aligned with the text input.

\begin{table*}[htbp]
    \caption{\footnotesize \textbf{Ablation analysis on HumanML3D test set of motion representation and generation heads.}}
    \centering
    \scalebox{0.72}{
    \begin{tabular}{c c c c c c c c c}
    \toprule
    \multicolumn{2}{c}{Representation} & \multicolumn{2}{c}{Dimension} 
    & \multicolumn{2}{c}{VQ Reconstruction} & \multicolumn{3}{c}{T2M Generation} \\
    \cline{1-2}
    \cline{3-4}
    \cline{5-6}
    \cline{7-9}
    \multicolumn{2}{c}{\footnotesize Type} & \multicolumn{2}{c}{\footnotesize Dim.} 
    & \footnotesize FID $\downarrow$ & \footnotesize Joint Pos. Err. $\downarrow$ 
    & \footnotesize FID $\downarrow$ & \footnotesize R-Precison@3 $\uparrow$ & \footnotesize Inference time $\downarrow$\\
    \midrule
    \multicolumn{2}{c}{Latent + Flow Matching} & \multicolumn{2}{c}{256} & 0.0054 & 0.0302 & \et{0.087}{.004} & \et{0.796}{.003} &  0.49 \\ 
    \multicolumn{2}{c}{Hybrid + Flow Matching} & \multicolumn{2}{c}{64, 192} & 0.0036 & 0.0234 & \et{0.061}{.003} & \et{0.832}{.002} &  0.52 \\ 
    \midrule
    \multicolumn{2}{c}{Latent + Diffusion} & \multicolumn{2}{c}{256} & 0.0054 & 0.0302 & \et{0.078}{.003} & \et{0.781}{.002} &  7.66 \\  
    \multicolumn{2}{c}{Hybrid + Diffusion} & \multicolumn{2}{c}{64, 192} & 0.0036 & 0.0234 & \et{0.075}{.003} & \et{0.801}{.002} &  8.02 \\ 
    \bottomrule
    \end{tabular}
    }
    \label{tab:abl2}
\end{table*}

\textbf{Flow Modeling Generation.} To assess the benefits of the flow modeling approach, we compared our model with a variant generated using the standard diffusion method. As shown in \Cref{tab:abl2}, under the same training iterations, the method using flow modeling significantly outperformed the diffusion model. Additionally, the inference speed of the flow modeling method was also significantly faster than that of the diffusion method.

\begin{table*}[htbp]
    \caption{\footnotesize \textbf{Quantitative evaluation on harder test set which selected from HumanML3D.}}
    \centering
    \scalebox{0.7}{
    \begin{tabular}{l c c c c}
    \toprule
     \multirow{2}{*}{Methods}  & \multicolumn{3}{c}{R-precision$\uparrow$} & \multirow{2}{*}{FID$\downarrow$} \\
    \cline{2-4}
     ~ & Top 1 & Top 2 & Top 3 \\
    \midrule
    MoMask~ & \et{0.467}{.002} & \et{0.655}{.003} & \et{0.748}{.003} & \etb{0.231}{.012} \\
    MoMask++~ & \et{0.473}{.003} & \et{0.669}{.003} & \et{0.769}{.002} & \et{0.354}{.015} \\
    MotionGPT3~ & \ets{0.496}{.002} & \ets{0.685}{.002} & \ets{0.787}{.003} & \et{0.658}{.032} \\
    \midrule
    Ours$^\text{nova}$~ & \et{0.479}{.002} & \et{0.678}{.003} & \et{0.776}{.002} & \et{0.349}{.014} \\
    Ours$^\text{orbit}$~ & \etb{0.502}{.003} & \etb{0.693}{.004} & \etb{0.796}{.002} & \ets{0.332}{.011} \\
    \bottomrule
    \end{tabular}
    }
    \label{tab:harder_hml}
\end{table*}

\subsection{Evaluation on Fine-Grained Text Understanding}
To rigorously assess our model's capability in comprehending complex and detailed language instructions, we conduct an attribute-specific analysis on a constructed ``Hard Test Set''. Specifically, we filter the standard HumanML3D test set to isolate prompts containing fine-grained constraints \cite{Fg_T2M,CoMo, GraphMotion}, such as directional indicators (e.g., left/right), rotational semantics (e.g., clockwise/counterclockwise), and precise numerical or degree-based instructions. The details split the hard test set are provided in Appendix \Cref{llm_split}. As detailed in \Cref{tab:harder_hml}, FlowCoMotion demonstrates a distinct advantage in these highly constrained scenarios. It achieves the best text-motion alignment on this subset. Additionally, FlowCoMotion maintains a competitive FID, significantly outperforming both MoMask++ and MotionGPT3 in motion quality under strict textual conditions. These quantitative results confirm that our approach effectively balances high-fidelity motion synthesis with precise semantic grounding, successfully addressing the complex attribute binding that frequently challenges existing methods.

\section{Conclusion}
In Conclusion, we propose FlowCoMotion, a novel generative flow matching modeling framework for text-driven 3D human motion generation. FlowCoMotion integrates a label-latent coupling method for accurate motion representation, ensuring semantic and dynamic properties, and a flow matching generation strategy for fast generation of high-quality motion. However, some limitations remain, such as the need for further optimization in generating longer sequences. Future work will explore zero-shot generalization capabilities.

\bibliographystyle{plain} 
\bibliography{reference}
\newpage
\appendix

 \section{More Results}
\subsection{More Qualitative Results}
\label{more_results}
We supplement some qualitative results of the text-to-motion comparison on HumanML3D, dividing them into shorter and longer prompts, as shown in \Cref{fig:supp} and \Cref{fig:supp2} respectively.

\textbf{For shorter prompts,} MotionGPT generally generates semantically correct motions, but it struggles with direction- and plural-related descriptions (e.g., ``right side'', ``arms'' and ``left/right arm''). While MoMask++ and MotionGPT3 improve in this regard, they still struggle to generate motions that perfectly match the text descriptions. Furthermore, the motions generated by MoMask++ are largely identical and lack dynamism. In contrast, our method is able to generate high-quality motions that are highly consistent with the input text.

\begin{figure*}[htbp]
    \centering
    \includegraphics[width=\linewidth]{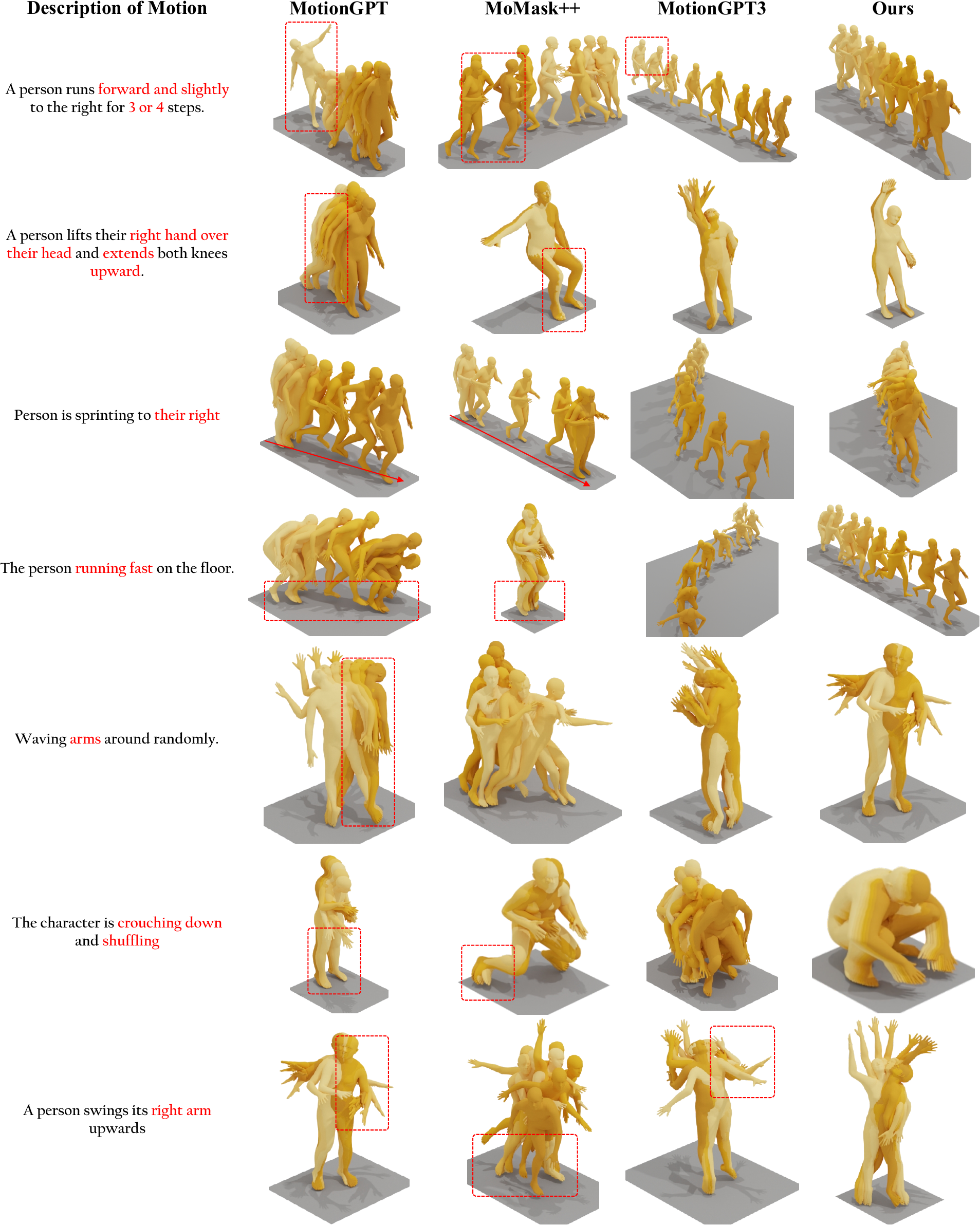}
    \caption{\footnotesize \textbf{Qualitative evaluation on text-to-motion generation.} The state-of-the-art methods presented in this paper are all implemented using the same inference and rendering settings on HumanML3D. Misaligned motion is indicated by red text and bounding boxes. We selected some results with shorter prompts. The results demonstrate that our motion-language framework exhibits strong text understanding capabilities, enabling the generation of more accurate and coherent motion.}
    \label{fig:supp}
\end{figure*}

\textbf{For longer prompts,} while MotionGPT produces semantically plausible basic motions, it struggles significantly with extended descriptions that require a series of distinct actions (e.g., knee down'' $\rightarrow$ raising right arm up'' $\rightarrow$ ``placing its forearm in front of the face''). Although MoMask++ and MotionGPT3 show improved sequence handling, they still fail to achieve precise text-motion alignment, frequently suffering from incorrect action ordering or the insertion of irrelevant movements. Conversely, our approach excels at interpreting complex, multi-action prompts, delivering high-quality motion sequences that strictly adhere to the textual conditions.

\begin{figure*}[htbp]
    \centering
    \includegraphics[width=\linewidth]{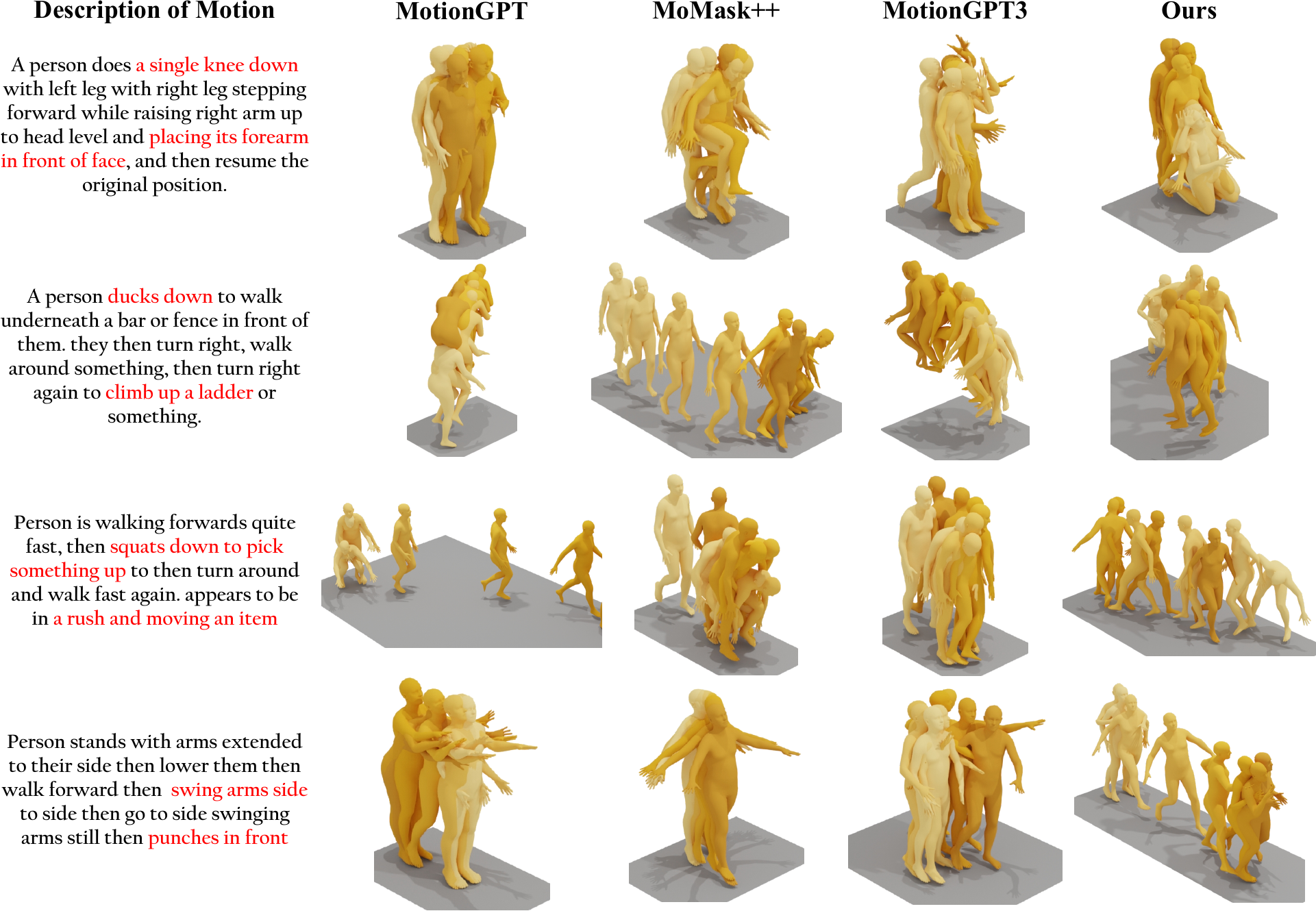}
    \caption{\footnotesize \textbf{Qualitative evaluation on text-to-motion generation.}  The provided state-of-the-art methods are under the same training and inference setting on HumanML3D. Misaligned motions are indicated by red texts. We selected some results with longer prompts. The results indicate that our motion-language framework exhibits strong text comprehension capabilities, leading to more accurate and coherent motion generation.}
    \label{fig:supp2}
\end{figure*}

\subsection{More Quantitative Evaluation}
\label{more_evaluation}
To comprehensively evaluate our method on text-to-motion generation, we report a full comparison with prior T2M approaches in \Cref{table_more}. The main methods include the following: TEMOS \cite{TEMOS}, TM2T \cite{TM2T}, T2M \cite{HumanML3D}, TM2D \cite{TM2D}, MotionDiffuse \cite{MotionDiffuse}, MDM \cite{mdm}, MotionLLM \cite{motionllm}, MLD \cite{MLD}, Motion Mamba \cite{Motion_Mamba}, ReMoDiffuse \cite{ReMoDiffuse}, Fg-T2M \cite{Fg_T2M}, T2M-GPT \cite{T2M-GPT}, EMDM \cite{EMDM}, GraphMotion \cite{GraphMotion}, AttT2M \cite{AttT2M}, GUESS \cite{GUESS}, MMM \cite{MMM}, MotionCLR \cite{motionclr}, MotionGPT \cite{MotionGPT}, MotionGPT-2 \cite{motiongpt2}, Parco \cite{ParCo}, StableMoFusion \cite{StableMoFusion}, DiverseMotion \cite{DiverseMotion}, BAD \cite{bad}, MoMask \cite{Momask}, Momask++ \cite{Snapmogen}, MotionGPT3 \cite{MotionGPT3}. The results show that our method consistently outperforms prior approaches and achieves competitive performance on HumanML3D. 

\begin{table*}[t]
    \caption{\footnotesize\textbf{Quantitative evaluation on \text{HumanML3D} test set.}}
    \centering
    \scalebox{0.75}{
    \begin{tabular}{l c c c c c c c}
    \toprule
     \multirow{2}{*}{Methods}  & \multicolumn{3}{c}{R-precision$\uparrow$} & \multirow{2}{*}{FID$\downarrow$} & \multirow{2}{*}{MultiModal Dist$\downarrow$} & \multirow{2}{*}{Diversity$\rightarrow$} & \multirow{2}{*}{MModality$\uparrow$}\\
    \cline{2-4}
     ~ & Top 1 & Top 2 & Top 3 \\
    \midrule
       Real motions & \et{0.511}{.003} & \et{0.703}{.003} & \et{0.797}{.002} & \et{0.002}{.000} & \et{2.974}{.008} & \et{9.503}{.065} & - \\
    \midrule
    TEMOS~ & \et{0.424}{.002} & \et{0.612}{.002} & \et{0.722}{.002} & \et{3.734}{.028} & \et{3.703}{.008} & \et{8.973}{.071} & \et{0.368}{.018} \\
    TM2T~ & \et{0.424}{.003} & \et{0.618}{.003} & \et{0.729}{.002} & \et{1.501}{.017} & \et{3.467}{.011} & \et{8.589}{.076} & \et{2.424}{.093} \\
    T2M~ & \et{0.457}{.002} & \et{0.639}{.003} & \et{0.740}{.003} & \et{1.067}{.002} & \et{3.340}{.008} & \et{9.188}{.002} & \et{2.090}{.083} \\
    TM2D~ & \et{0.319}{.000} & - & - & \et{1.021}{.000} & \et{4.098}{.000} & \etb{9.513}{.000} & \etb{4.139}{.000} \\
    MotionDiffuse~ & \et{0.491}{.001} & \et{0.681}{.001} & \et{0.782}{.001} & \et{0.630}{.001} & \et{3.113}{.001} & \et{9.410}{.049} & \et{1.553}{.042} \\
    MDM~ & \et{0.320}{.005} & \et{0.498}{.004} & \et{0.611}{.007} & \et{0.544}{.044} & \et{5.566}{.027} & \et{9.559}{.086} & \et{2.799}{.072} \\
    MotionLLM~ & \et{0.482}{.004} & \et{0.672}{.003} & \et{0.770}{.002} & \et{0.491}{.019} & \et{3.138}{.010} & \et{9.838}{.244} & - \\
    MLD~ & \et{0.481}{.003} & \et{0.673}{.003} & \et{0.772}{.002} & \et{0.473}{.013} & \et{3.196}{.010} & \et{9.724}{.082} & \et{2.413}{.079} \\
    Motion Mamba~ & \et{0.502}{.003} & \et{0.693}{.002} & \et{0.792}{.002} & \et{0.281}{.009} & \et{3.060}{.058} & \et{9.871}{.084} & \et{2.294}{.058} \\
    MotionCraft~ & \et{0.501}{.003} & \et{0.697}{.003} & \et{0.796}{.002} & \et{0.173}{.002} & \et{3.025}{.008} & \et{9.543}{.098} & - \\
    ReMoDiffuse~ & \et{0.510}{.005} & \et{0.698}{.006} & \et{0.795}{.004} & \et{0.103}{.004} & \et{2.974}{.016} & \et{9.018}{.075} & \et{1.795}{.043} \\
    Fg-T2M~ & \et{0.492}{.002} & \et{0.683}{.003} & \et{0.783}{.002} & \et{0.243}{.019} & \et{3.109}{.007} & \et{9.278}{.072} & \et{1.614}{.049} \\
    T2M-GPT~ & \et{0.492}{.003} & \et{0.679}{.002} & \et{0.775}{.002} & \et{0.141}{.005} & \et{3.121}{.009} & \et{9.722}{.082} & \et{1.831}{.048} \\
    EMDM~ & \et{0.498}{.007} & \et{0.684}{.006} & \et{0.786}{.006} & \et{0.112}{.019} & \et{3.110}{.027} & \et{9.551}{.078} & \et{1.641}{.078} \\
    GraphMotion~ & \et{0.504}{.003} & \et{0.699}{.002} & \et{0.785}{.002} & \et{0.116}{.007} & \et{3.070}{.008} & \et{9.692}{.067} & \et{2.766}{.096} \\
    AttT2M~ & \et{0.499}{.003} & \et{0.690}{.002} & \et{0.786}{.002} & \et{0.112}{.006} & \et{3.038}{.007} & \et{9.700}{.090} & \et{2.452}{.051} \\
    GUESS~ & \et{0.503}{.003} & \et{0.688}{.002} & \et{0.787}{.002} & \et{0.109}{.007} & \et{3.006}{.007} & \et{9.826}{.104} & \et{2.430}{.100} \\
    MMM~ & \et{0.504}{.003} & \et{0.696}{.003} & \et{0.794}{.002} & \et{0.080}{.003} & \et{2.998}{.007} & \et{9.411}{.058} & \et{1.164}{.041} \\
    MotionCLR~ & \et{0.542}{.001} & \et{0.733}{.002} & \et{0.827}{.003} & \et{0.099}{.003} & \et{2.981}{.011} & - & \et{2.145}{.043} \\
    MotionGPT~ & \et{0.492}{.003} & \et{0.681}{.003} & \et{0.778}{.002} & \et{0.232}{.008} & \et{3.096}{.008} & \et{9.528}{.071} & \et{2.008}{.084} \\
    MotionGPT-2~ & \et{0.496}{.002} & \et{0.691}{.003} & \et{0.782}{.004} & \et{0.191}{.004} & \et{3.080}{.013} & \et{9.860}{.026} & \et{2.137}{.022} \\
    ParCo~ & \et{0.515}{.003} & \et{0.706}{.003} & \et{0.801}{.002} & \et{0.109}{.005} & \et{2.927}{.008} & \et{9.576}{.088} & \et{1.382}{.060} \\
    StableMoFusion~ & \et{0.553}{.003} & \et{0.748}{.002} & \et{0.841}{.002} & \et{0.098}{.003} & - & \et{9.748}{.092} & \et{1.774}{.051} \\
    DiverseMotion~ & \et{0.515}{.003}  & \et{0.706}{.002} & \et{0.802}{.002} & \et{0.072}{.004} & \et{2.941}{.007} & \et{9.683}{.102} & \et{1.869}{.089} \\
    BAD~ & \et{0.517}{.002} & \et{0.713}{.003} & \et{0.808}{.003} & \et{0.065}{.003} & \et{2.901}{.008} & \et{9.694}{.068} & \et{1.194}{.044} \\
    MoMask~ & \et{0.521}{.002} & \et{0.713}{.002} & \et{0.807}{.002} & \etb{0.045}{.002} & \et{2.958}{.008} & \et{9.620}{.064} & \et{1.241}{.04} \\
    MoMask++~ & \et{0.528}{.003} & \et{0.718}{.003} & \et{0.811}{.002} & \et{0.072}{.003} & \et{2.912}{.008} & \et{9.573}{.062} & \et{1.227}{.046} \\
    MotionGPT3~ & \ets{0.543}{.003} & \ets{0.735}{.002} & \ets{0.828}{.002} & \et{0.217}{.010} & \etb{2.793}{.007} & \et{9.662}{.072} & \et{1.366}{.046} \\
    \midrule
    Ours$^\text{nova}$  & \et{0.532}{.002} & \et{0.725}{.002} & \et{0.819}{.002} & \et{0.072}{.003} & \et{2.887}{.006} & \et{9.377}{.078} & \et{1.875}{.088} \\
    Ours$^\text{orbit}$  & \etb{0.550}{.003} & \etb{0.740}{.003} & \etb{0.832}{.002} & \ets{0.061}{.003} & \ets{2.852}{.008} & \ets{9.457}{.067} & \ets{2.071}{.093} \\
    \bottomrule
    \end{tabular}
    }
    \label{table_more}
\end{table*}

\section{More Ablation Studies}
\label{appden:abl}
In this section, we have made additional supplementary information as follows: (1). The impact of the number of ODE sampling steps and the size of the CFG model during the flow matching process on the generated results. (2). We compared the reasoning speed of our framework with that of other competitive frameworks. (3). The influence of the proportions of Tokens and Latents on motion reconstruction.

\begin{figure}[htbp]
    \centering
    \includegraphics[width=0.75\linewidth]{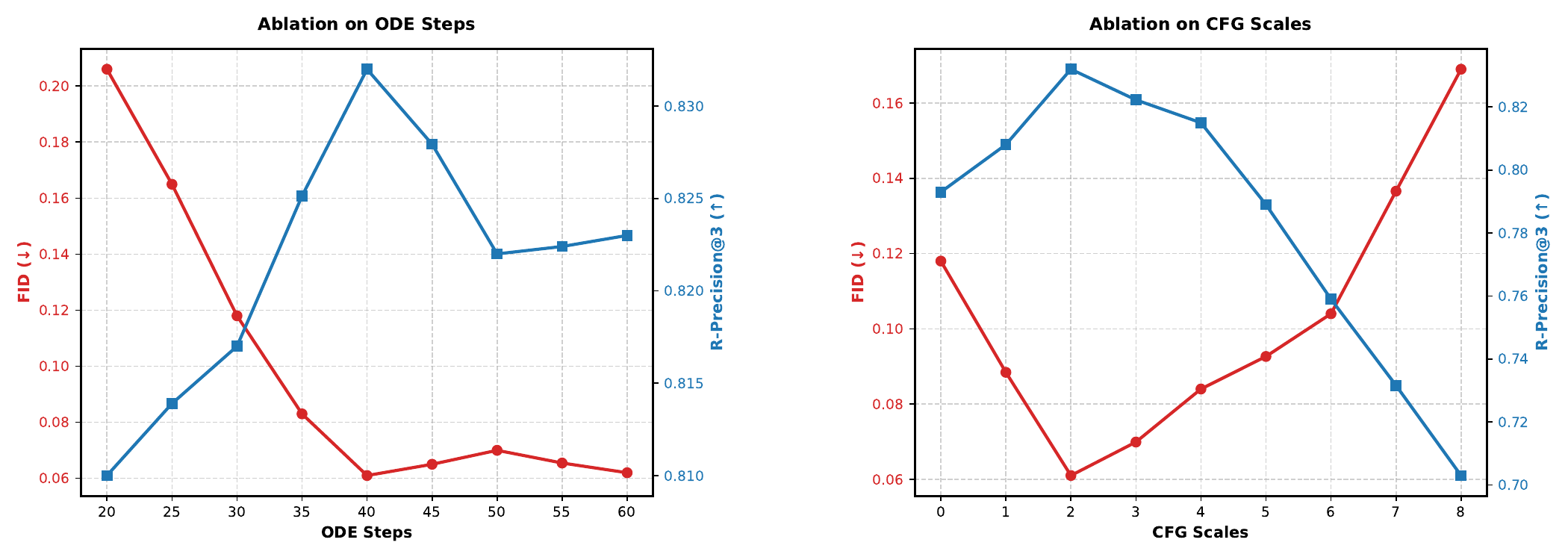}
    \caption{\footnotesize Evaluation sweep over ODE numbers L (left) and guidance scale s (right) in inference. We find a accuracy-fidelity sweep spot around s = 2, meanwhile 40 iterations (L = 40) for flow matching yield sufficiently good results.}
    \label{fig:ode}
\end{figure}

\textbf{Inference Hyperparameters.} During the inference process of flow matching, the scaling coefficient s of CFG and the number of ODE sampling L are two crucial hyperparameters. In \Cref{fig:ode}, we present the performance curves of FID and R-Precision@3 by changing different values of s and L. Several key observations emerge. Firstly, more sampling times do not necessarily lead to better results. As L increases, FID quickly converges to a minimum value, typically within approximately 40 iterations. After more than 40 iterations, there is no further performance improvement in FID, but R-Precision@3 will experience a certain decline. Secondly, the optimal guidance coefficient s for inference is approximately around s = 2. Over-guided decoding may even lead to a decrease in performance. Our FlowCoMotion requires fewer inference steps than most diffusion models.

\begin{figure}[htbp]
    \centering
    \includegraphics[width=0.55\linewidth]{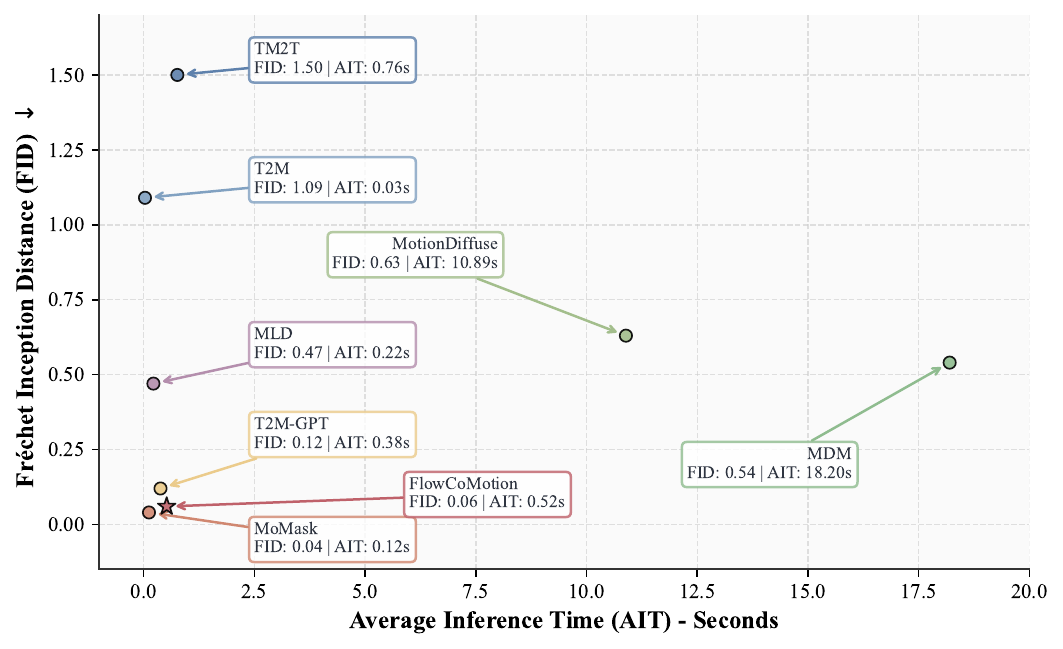}
    \caption{\footnotesize \textbf{Comparison of inference time costs.} All tests were conducted on a single NVIDIA RTX 2080 Ti. Models closer to the origin indicate better performance. For non-autoregressive models, our method significantly improves inference speed, thanks to the flow matching directional velocity field.}
    \label{fig:inference}
\end{figure}

\textbf{Inference Time.} In \Cref{fig:inference}, we compare inference efficiency and quality across methods, measuring runtime as the average latency over 100 samples on an NVIDIA RTX 2080 Ti. FlowCoMotion provides a stronger quality–efficiency trade-off than prior approaches. Specifically, by leveraging a Flow Matching framework rather than traditional diffusion models, our approach significantly accelerates the generation process. The optimal transport trajectories inherent in Flow Matching allow for a drastic reduction in the number of sampling steps required without compromising motion fidelity. Consequently, FlowCoMotion substantially decreases inference latency, achieving generation speeds that closely approach those of fast autoregressive-based methods, while still maintaining superior motion quality.

\begin{figure}[htbp]
    \centering
    \includegraphics[width= 0.55\linewidth]{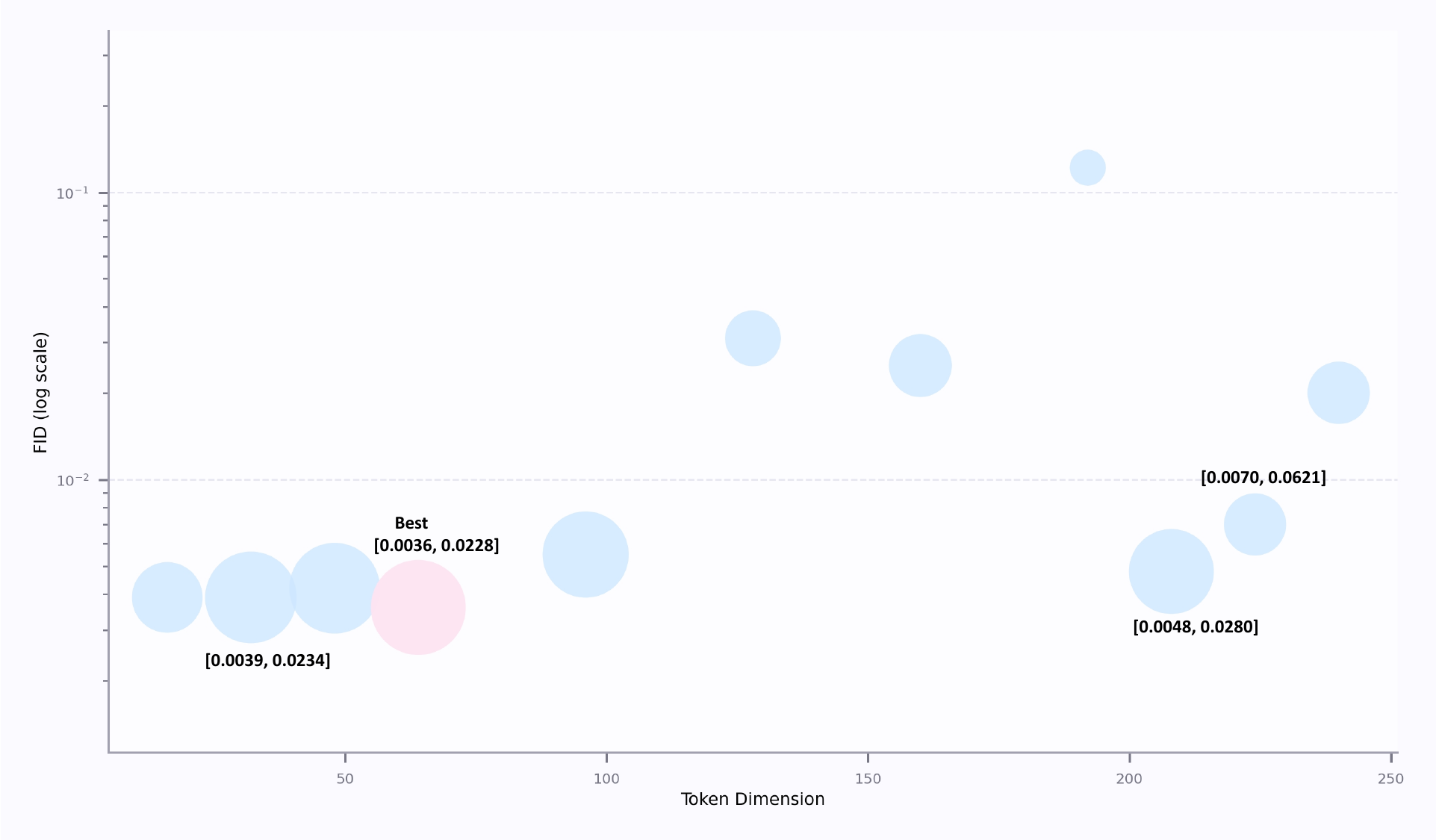}
    \caption{\footnotesize \textbf{Ablation analysis on representation dimension choices.} We report reconstruction quality using FID and MPJPE. The larger of the circles represents a better MPJPE.}
    \label{fig:dim}
\end{figure}

\textbf{Representation Dimension.} We study how allocating capacity between the token and latent branches affects reconstruction quality. Keeping the total dimensionality fixed, we vary the dimensions of both branches and summarize the results in \Cref{fig:dim}. Performance is best when one branch receives most of the capacity while the other retains a non-trivial share; overly skewed allocations lead to clear degradation. This trend supports the complementarity of the two representations: tokens capture semantic structure, whereas continuous latent preserve geometric and temporal detail. A primary–secondary allocation thus provides the strongest hybrid representation under a fixed budget. Conversely, forcing an equal division restricts the representational power of both branches, preventing either from fully modeling its target features. Based on these empirical observations, we adopt this optimal asymmetric allocation as our default configuration for all subsequent experiments.

\section{Human Evaluation on Text-to-Motion Quality}
This evaluation presents a comprehensive user evaluation of our motion generation framework FlowCoMotion by conducting a Google Forms survey as shown in \Cref{fig:user_study_gui}. Participants were asked to evaluate the effectiveness and realism of the generated motions based on accuracy, smoothness, realism, and overall user experience.

\textbf{Evaluation Setup.} 1. For the individual evaluation, we provided 44 distinct motion sequences. To mitigate potential biases introduced by user selection inertia, 6 additional sequences were randomly inserted into the set but excluded from the final analysis. Participants rated these motions on a three-point scale (scored from 1 to 3), and we report the final average rating. 2. For the comparative evaluation, we curated 50 sets of motion sequences generated by different models based on identical textual descriptions. Participants were asked to select the best-performing model within each set. This allowed for a direct assessment of our model's realism, coherence, and engagement against existing alternatives. We report the final average win rate (out of 50) based on how often users felt our method achieved the best results. 

\begin{figure*}[htbp]
    \centering
    \includegraphics[width=\linewidth]{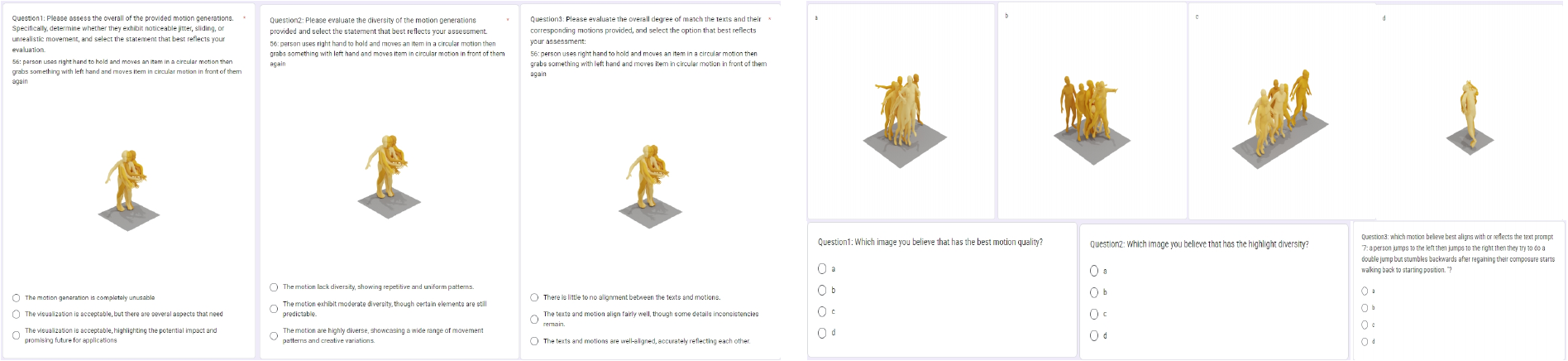}
    \caption{\footnotesize \textbf{User study form.} The User Interface (UI) was used in our user study. Our user study consists of 50 individual evaluations and 50 comparative evaluations. We incorporate elements from other methods to prevent the review results from becoming invalid due to user inertia.}
    \label{fig:user_study_gui}
\end{figure*}

\textbf{Baseline.} We selected three equally competitive models: MotionGPT, Momask++, and MotionGPT3. We evaluated them using high-quality rendered sequence images and videos.

\textbf{Human Evaluation Results Analysis.} As illustrated in \Cref{fig:user_study_results}, FlowCoMotion demonstrates clear superiority over baselines including MoMask++ , MotionGPT3, and MotionGPT. In the individual evaluation, our method achieved high satisfaction rates of 86.0\% , 88.0\% , and 89.0\% for Motion Quality , Motion Diversity , and Text-Motion Alignment, respectively. Furthermore, the comparative evaluation reveals a strong human preference for FlowCoMotion, which was selected as the top-performing model by 44.8\% , 48.8\% , and 52.3\%  of participants across the three respective metrics. These findings confirm that our model generates more realistic, diverse, and well-aligned motions than existing approaches.

\begin{figure*}[htbp]
    \centering
    \includegraphics[width=\linewidth]{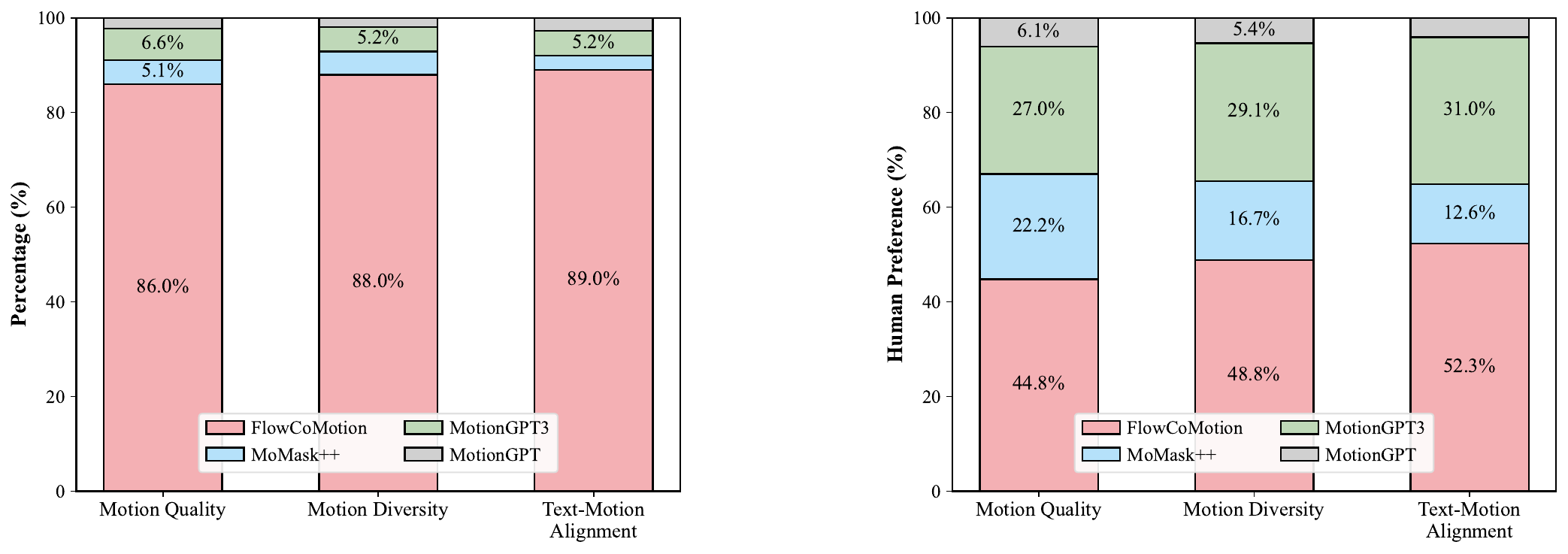}
    \caption{\footnotesize Human preference on Text-to-Motion Generation by comparing FlowCoMotion with MotionGPT \cite{MotionGPT}, Momask++ \cite{Snapmogen} and MotionGPT3 \cite{MotionGPT3}. We report the scores on three generation aspects.}
    \label{fig:user_study_results}
\end{figure*}

\section{Why Token-Latent Coupling Helps: A Theoretical View}
\label{sec:theory_coupling}
Let $X$ represent a motion sequence, modeled as a vector in $\mathbb R^{T\times d}$, where $T$ denotes the sequence length and $d$ is the dimension of each representation. Additionally, let $C$ be the textual condition that drives the motion generation process.

Our framework constructs two distinct representations from $X$: (1). A discrete token representation $Z_{q}$, derived from the Residual Vector Quantization tokenizer. (2). A continuous latent representation $Z_{c}$, obtained from a VAE's encoder.

To integrate both modalities, we first extract a dequantized multi-scale token feature $\hat{z}_{q}=\sum_{v=0}^V\mathcal{I}(\hat{f}^v,h^v)$, where $\hat{f}^v$ represents the token features at scale $v$, and $h^v$ denotes the associated multi-scale feature for each token. We can combine $\hat z_{q}$ with $\hat z_{c}$ using a coupling network, producing a fused representation $\mathcal C$ defined as:
\begin{equation}
\hat{\mathbf{z}}_{gt}=\mathcal C(\hat{z}_{q}, \hat z_{c}) \quad \hat{\mathbf{m}}=\mathcal D(\hat{\mathbf{z}}_{gt})
\end{equation}
In Stage 3, we define the flow matching data endpoint as $z_{1}=z_{gt}$, and train a condition velocity field to map $z_{0} \sim \mathcal N(0, I) $ to $z_{1}$, where $z_{0}$ is drawn from a standard Gaussian distribution.

\subsection{Coupling improves the best achievable reconstruction risk.}
We quantify the utility of token--latent coupling through the population Bayes risk under squared reconstruction loss. For any representation $S$ of $X$, define the optimal reconstruction risk
\begin{equation}
\mathcal{R}^\star(S)\triangleq \inf_{g}\; \mathbb{E}\!\left[\|X-g(S)\|_2^2\right],
\end{equation}
where $g$ ranges over all measurable decoders.

\begin{proposition}[Bayes risk monotonicity and a joint-representation limit]
\label{prop:bayes_monotone}
Let $S_1,S_2$ be two representations of $X$. If $\sigma(S_{1}) \subseteq \sigma(S_{2})$ (i.e., $S_2$ is at least as informative as $S_1$), then
\begin{equation}
\mathcal{R}^\star(S_2)\le \mathcal{R}^\star(S_1).
\end{equation}
In particular, for the joint-representation $S_{\text{joint}}\triangleq (\hat z_c,\hat z_q)$, we have
\begin{equation}
\mathcal{R}^\star(\hat z_c,\hat z_q)\;\le\;\min\Big\{\mathcal{R}^\star(\hat z_c),\; \mathcal{R}^\star(\hat z_q)\Big\}.
\label{eq:joint_advantage}
\end{equation}
Moreover, the inequality in \Cref{eq:joint_advantage} is typically strict whenever the two branches are complementary, i.e., one branch reduces the conditional reconstruction uncertainty left by the other on a set of nonzero probability mass.
\end{proposition}

\begin{proof}
Under squared loss, the Bayes-optimal decoder is the conditional expectation $g^\star(S)=\mathbb{E}[X\mid S]$. Hence
\begin{equation}
\mathcal{R}^\star(S)
=\mathbb{E}\!\left[\|X-\mathbb{E}[X\mid S]\|_2^2\right]
=\mathbb{E}\!\left[\mathrm{Var}(X\mid S)\right].
\end{equation}
If $\sigma(S_1)\subseteq\sigma(S_2)$, conditioning on $S_2$ cannot increase conditional variance, implying $\mathrm{Var}(X\mid S_2)\le \mathrm{Var}(X\mid S_1)$ almost surely, and thus $\mathcal{R}^\star(S_2)\le \mathcal{R}^\star(S_1)$. Applying this with $S_2=(\hat z_c,\hat z_q)$ and $S_1=\hat z_c$ (or $S_1=\hat z_q$) yields \Cref{eq:joint_advantage}.
\end{proof}

\Cref{prop:bayes_monotone} should be read as a representational \emph{limit statement}: the joint statistic $(\hat z_c,\hat z_q)$ admits a Bayes-optimal reconstruction risk no worse than either branch alone. Our coupling network $\hat z_{gt}=\mathcal{C}(\hat z_q,\hat z_c)$ can be viewed as a learnable \emph{compression} of this joint statistic. In Stage 1 and 2, we train $\mathcal{C}$ and the decoder $\mathcal{D}$ to reconstruct $X$ from $\hat z_{gt}$, which encourages $\hat z_{gt}$ to preserve the reconstruction-relevant information in $(\hat z_c,\hat z_q)$. Consequently, as the capacity of $\mathcal{C}$ and $\mathcal{D}$ increases and optimization succeeds, the achievable reconstruction risk using $\hat z_{gt}$ can approach the joint-representation limit in \Cref{eq:joint_advantage}.


\subsection{Coupling Reduces the Irreducible Error of Flow Matching.}
We next connect coupling to conditional flow matching. Let $z_1$ denote the target endpoint latent (in our method, $z_1=\hat z_{gt}$), and draw $z_0\sim \mathcal{N}(0,I)$. Consider the standard linear interpolation path
\begin{equation}
z_t \;=\; (1-t)z_0+t z_1,\qquad u_t \;\triangleq\; \frac{d z_t}{dt} \;=\; z_1-z_0,\qquad t\sim \mathcal{U}(0,1).
\end{equation}
The population conditional flow matching objective is
\begin{equation}
\mathcal{L}_{\mathrm{FM}}(v)\triangleq \mathbb{E}\!\left[\|u_t - v(z_t,t,C)\|_2^2\right],
\end{equation}
where the expectation is over \((t,z_0,z_1,C)\). By the optimality of \(L_2\) regression,
\begin{equation}
v^\star(z_t,t,C)=\mathbb{E}[u_t\mid z_t,t,C],
\end{equation}
and the minimum achievable loss equals the irreducible conditional variance
\begin{equation}
\mathcal{L}_{\mathrm{FM}}^\star(z_1)
\;\triangleq\; \inf_v \mathcal{L}_{\mathrm{FM}}(v)
\;=\; \mathbb{E}\!\left[\mathrm{Var}(u_t\mid z_t,t,C)\right].
\label{eq:fm_irreducible}
\end{equation}

We compare two endpoint constructions under the same text condition  $C$: a coupled endpoint $z_1^{\mathrm{cp}}$ (e.g., $z_1^{\mathrm{cp}}=\hat z_{gt}=\mathcal{C}(\hat z_q,\hat z_c)$) and a single-branch endpoint $z_1^{\mathrm{sb}}$. These two choices induce different random variables $z_t^{\mathrm{cp}}=(1-t)z_0+t z_1^{\mathrm{cp}}$ and $z_t^{\mathrm{sb}}=(1-t)z_0+t z_1^{\mathrm{sb}}$, and their irreducible errors are compared through their respective conditionings. In addition, throughout Stage 1 and 2 we regularize the latent scales (e.g., VAE KL regularization), so $z_1^{\mathrm{cp}}$ and $z_1^{\mathrm{sb}}$ are comparable in magnitude; hence the variance change discussed below is attributed to reduced conditional ambiguity (mode dispersion) rather than trivial rescaling.

\begin{theorem}[Reduction of an irreducible error upper bound via endpoint concentration]
\label{thm:fm_bound_decrease}
Assume that coupling makes the endpoint distribution more concentrated under the text condition, in the sense that
\begin{equation}
\mathbb{E}\!\left[\|z_1^{\mathrm{cp}}-\mathbb{E}[z_1^{\mathrm{cp}}\mid C]\|_2^2\right]
\;\le\;
\mathbb{E}\!\left[\|z_1^{\mathrm{sb}}-\mathbb{E}[z_1^{\mathrm{sb}}\mid C]\|_2^2\right].
\label{eq:endpoint_trace}
\end{equation}
Then, for any fixed $t\in(0,1)$, the irreducible conditional variance term admits the bound
\begin{equation}
\mathbb{E}\!\left[\mathrm{Var}(u_t^{\mathrm{cp}}\mid z_t^{\mathrm{cp}},t,C)\right]
\;\le\; \frac{1}{(1-t)^2}\;
\mathbb{E}\!\left[\|z_1^{\mathrm{cp}}-\mathbb{E}[z_1^{\mathrm{cp}}\mid C]\|_2^2\right],
\label{eq:cp_bound}
\end{equation}
and analogously
\begin{equation}
\mathbb{E}\!\left[\mathrm{Var}(u_t^{\mathrm{sb}}\mid z_t^{\mathrm{sb}},t,C)\right]
\;\le\; \frac{1}{(1-t)^2}\;
\mathbb{E}\!\left[\|z_1^{\mathrm{sb}}-\mathbb{E}[z_1^{\mathrm{sb}}\mid C]\|_2^2\right].
\label{eq:sb_bound}
\end{equation}
Consequently, under \Cref{eq:endpoint_trace}, coupling reduces an \emph{upper bound} on the irreducible variance term in \Cref{eq:fm_irreducible}. (The factor $\tfrac{1}{(1-t)^2}$ diverges as $t\to 1$, which is expected: near the endpoint, the uncertainty of the velocity field is highly sensitive to the endpoint distribution.)
\end{theorem}

\begin{proof}
Fix $t\in(0,1)$ and condition on $(t,C)$. Since $z_t=(1-t)z_0+t z_1$ and $u_t=z_1-z_0$, we eliminate $z_0$ via
\begin{equation}
z_0=z_1-u_t
\quad\Longrightarrow\quad
z_t=(1-t)(z_1-u_t)+t z_1=z_1-(1-t)u_t,
\end{equation}
equivalently $z_1=z_t+(1-t)u_t$. Conditioning on $(z_t,t,C)$ yields
\begin{equation}
\mathrm{Var}(z_1\mid z_t,t,C)=(1-t)^2\,\mathrm{Var}(u_t\mid z_t,t,C),
\end{equation}
and hence
\begin{equation}
\mathbb{E}\!\left[\mathrm{Var}(u_t\mid z_t,t,C)\right]
=\frac{1}{(1-t)^2}\,\mathbb{E}\!\left[\mathrm{Var}(z_1\mid z_t,t,C)\right].
\end{equation}
By the law of total variance,
\begin{equation}
\begin{aligned}
\mathbb{E}\!\left[\mathrm{Var}(z_1\mid z_t,t,C)\right]
&\le \mathbb{E}\!\left[\mathrm{Var}(z_1\mid C)\right] \\
&= \mathbb{E}\!\left[\|z_1-\mathbb{E}[z_1\mid C]\|_2^2\right]
-\mathbb{E}\!\left[\|\mathbb{E}[z_1\mid z_t,t,C]-\mathbb{E}[z_1\mid C]\|_2^2\right] \\
&\le \mathbb{E}\!\left[\|z_1-\mathbb{E}[z_1\mid C]\|_2^2\right].
\end{aligned}
\end{equation}
Combining the above gives the generic bound
\begin{equation}
\mathbb{E}\!\left[\mathrm{Var}(u_t\mid z_t,t,C)\right]
\le \frac{1}{(1-t)^2}\;
\mathbb{E}\!\left[\|z_1-\mathbb{E}[z_1\mid C]\|_2^2\right],
\end{equation}
which, applied to $z_1^{\mathrm{cp}}$ and $z_1^{\mathrm{sb}}$, yields \Cref{eq:cp_bound,eq:sb_bound}. Finally, \Cref{eq:endpoint_trace} transfers the improvement to the corresponding upper bound.
\end{proof}

\Cref{eq:fm_irreducible} shows that conditional flow matching is fundamentally limited by the conditional dispersion of the target velocity $u_t$ given $(z_t,t,C)$. In text-to-motion, the mapping $C\mapsto X$ is typically one-to-many, so a single-branch endpoint may aggregate multiple semantic modes under the same caption, increasing the variability of $z_1$ and injecting additional uncertainty into the regression target $u_t$. Coupling injects discrete semantic anchors into the endpoint construction, which can make $p(z_1\mid C)$ more concentrated in the sense of \Cref{eq:endpoint_trace}, thereby reducing an upper bound on the irreducible variance term in \Cref{eq:fm_irreducible} and easing the conditional regression task learned by $v_\theta$.

\subsection{Discussion on assumptions and how our design satisfies them.}
The above analysis relies on two practical assumptions. First, the coupled representation should be sufficiently \emph{information-preserving} with respect to the two branches, so that fusing $(\hat z_q,\hat z_c)$ into $\hat z_{gt}$ does not discard reconstruction-relevant cues. In our design, this is supported by an expressive coupling network $\mathcal{C}$ (e.g., residual MLP/Transformer blocks) and the reconstruction-driven supervision in Stage 1 and 2, where $\mathcal{C}$ and $\mathcal{D}$ are optimized to minimize $\mathbb{E}\|X-\mathcal{D}(\mathcal{C}(\hat z_q,\hat z_c))\|_2^2$. This training objective encourages $\hat z_{gt}$ to behave as a faithful compression of the joint statistic, aligning with the representational limit in \Cref{eq:joint_advantage}.

Second, the variance-bound improvement in \Cref{thm:fm_bound_decrease} presumes that coupling makes the endpoint distribution $p(z_1\mid C)$ more concentrated in a scale-comparable space. Our implementation promotes this comparability through latent regularization and normalization (e.g., the VAE prior/KL regularization for $\hat z_c$ and codebook/embedding normalization for the RVQ branch), which prevents the bound reduction from being explained by trivial rescaling. Moreover, the RVQ commitment loss stabilizes token assignments across time scales, reducing semantic drift of $\hat z_q$ and improving the predictability between discrete anchors and continuous motion dynamics. Empirically, this reduces conditional ambiguity of the endpoint under the same caption, making the conditional regression target in flow matching less noisy and easier to learn with finite-capacity vector fields.

\section{Details on Our Work}
\label{More_details}
\subsection{Training Objectives}

\textbf{Reconstruction Loss.} As the primary objective, this loss ensures the accurate recovery of motion sequences by minimizing the $\texttt{L2}$ distance between predicted ($\hat{X}$) and ground truth ($X$) motions. Its weight is fixed at $\lambda_{rec} = 1.0$ across both stages:
\begin{equation}
L_{rec} = \mathbb{E}[ \| X - \hat{X} \|_2^2 ]
\end{equation}

\textbf{Commitment Loss.} To stabilize the quantization process, the commitment loss ($\lambda_{commit} = 0.01$) minimizes the distance between the encoder's continuous latent output and its nearest codebook entry. Coupled with temporal resolution sampling, this guarantees that the model effectively samples semantic anchors at varying time intervals without needing explicit formulation here.

\textbf{Forward Kinematics (FK) Loss.} Acting as a structural regularization term to ensure physically plausible motions, the FK loss compares predicted and expected joint positions:
\begin{equation}
L_{fk} = \mathbb{E}[ \| X_{\text{pred}} - X_{\text{FK}} \|_2^2 ]
\end{equation}
We set $\lambda_{fk} = 0.5$ in Stage 1 and reduce it to $0.1$ in Stage 2.

\textbf{Dynamic Distillation Loss.} In Stage 1, a multi-view teacher-student alignment regularizes the latent space. The weight $\lambda_{dis}$ linearly warms up to $0.02$, holds constant, then decays to $0.0$. 
\emph{View generation.} Sequences yield two global views and four local views (random crops $\in [0.3, 0.7]$). Local views receive temporal augmentations: time-warp, Gaussian noise, and temporal masking.
\emph{Aggregation.} A single per-view representation is obtained via temporal masked mean pooling:
\begin{equation}
\phi(F,M)=\frac{\sum_{t=1}^{T} M_t\,F_t}{\sum_{t=1}^{T} M_t+\varepsilon}
\end{equation}
\emph{Temperature and centering.} We use separate temperatures ($\tau_{\text{teacher}}=0.04$, $\tau_{\text{student}}=0.07$) and center teacher logits via an EMA buffer. The $\texttt{L2}$ distillation objective minimizes the error between student outputs and teacher targets.

\textbf{Velocity, Acceleration, and Jerk Losses.} Applied in Stage 2 to enhance temporal consistency, these losses penalize deviations in the first three derivatives of motion: velocity ($\lambda_{vel} = 0.2$), acceleration ($\lambda_{acc} = 0.05$), and jerk ($\lambda_{jerk} = 0.01$). Balancing these higher-order kinematic constraints ensures smooth, natural transitions and overall generation stability.

\subsection{Codebook Size and Resolution Segmentation}
The codebook size and resolution segmentation significantly impact quantization and generation quality. Our model fixes the codebook size at $1024$ with an entry dimension of $64$. Combined with a continuous latent dimension of $192$, the total latent dimension per sample is $256$.To optimize quantization, we adapt the resolution segmentation per dataset. For HumanML3D, we apply a four-layer scale $\mathcal{S}_h = [8, 4, 2, 1]$. In SnapMoGen, we utilize a two-layer scale $\mathcal{S}_h = [2, 1]$. Through temporal resolution sampling, the Token Branch extracts semantic anchors at varying intervals to enhance fidelity. Combining these varying scales with the commitment loss enables the model to capture multiple levels of motion granularity, producing detailed and coherent sequences across diverse datasets.

\subsection{Flow Matching Details}
In our implementation, we define a continuous path between a prior latent variable $z_0 \sim \mathcal{N}(0, I)$ and the real data latent $z_1$, leveraging a velocity field predicted by the flow model.

\textbf{Latent Variable Interpolation and Velocity.} During training and inference, the latent variable $z_t$ is linearly interpolated as $z_t = t z_1 + (1 - t) z_0$ for $t \in [0, 1]$, smoothly transitioning from Gaussian noise to the data distribution. The conditional flow model $v_\theta(z_t, t, c)$ predicts the instantaneous velocity $u_t = z_1 - z_0$ at each timestep, utilizing $z_t$ and an auxiliary textual condition $c$.

\textbf{Flow Matching Loss.} The training objective minimizes the mean squared error between the predicted and true velocities:
\begin{equation}
\mathcal{L} = \mathbb{E}_{z_0, z_1, t} \left[ \left\| (z_1 - z_0) - v_\theta(z_t, t, c) \right\|_2^2 \right],
\end{equation}
where $z_1$ is the ground truth motion latent from the token-latent coupling network. This enables the model to accurately bridge the prior and data distributions.

\textbf{Inference and ODE Solver.} Inference begins with $z_0 \sim \mathcal{N}(0, I)$. We solve the equation $\frac{d z_t}{dt} = v_\theta(z_t, t, c)$ over $t \in [0, 1]$ using a fixed-step Euler solver to output the final latent $\hat{z}_1$. We set the total steps to $40$, updating the latent iteratively via $z_{t+1} = z_t + \Delta t \cdot v_\theta(z_t, t, c)$. This step configuration provides an optimal trade-off between high-fidelity motion quality and computational efficiency.

\textbf{Classifier-Free Guidance (CFG).} We apply CFG with a fixed scale $s=2.0$. By concurrently predicting conditional ($v_{\text{cond}}$) and unconditional ($v_{\text{uncond}}$, via zeroed text embeddings) velocities, we guide the generation:
\begin{equation}
v_{\text{cfg}} = v_{\text{uncond}} + s \cdot (v_{\text{cond}} - v_{\text{uncond}}).
\end{equation}

\textbf{Noise Schedule.} By modulating the velocity field $v_\theta(z_t, t, c)$ with time $t$, the noise schedule ensures the generated sequence is progressively refined.

\begin{figure*}[htbp]
    \centering
    \includegraphics[width=\linewidth]{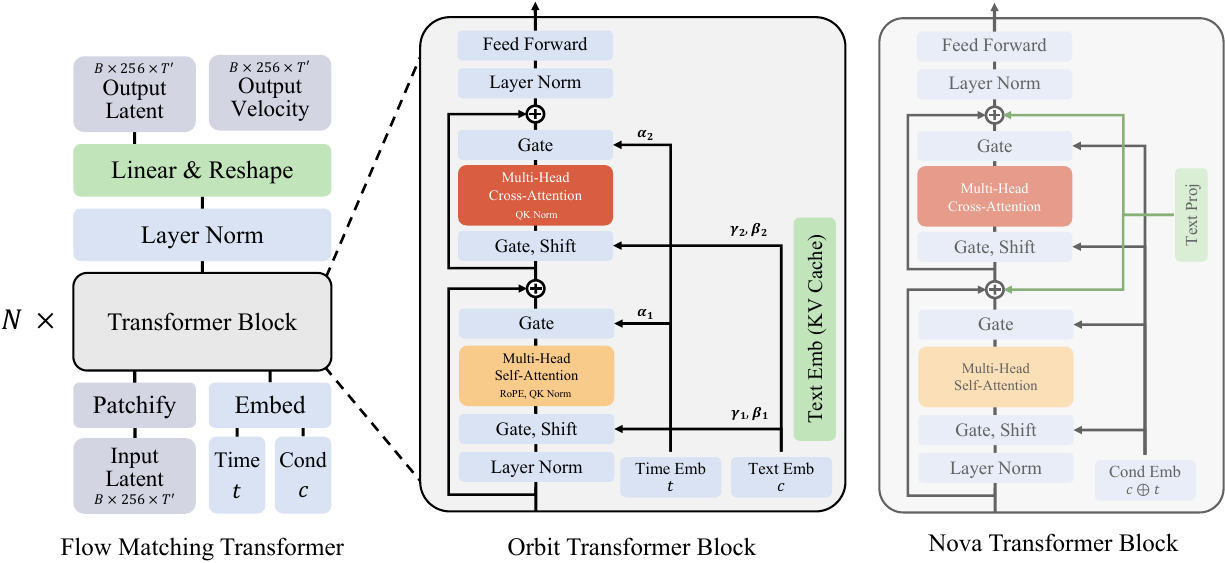}
    \caption{\footnotesize \textbf{Details of our Flow Matching Transformer model.} Given an Input Latent of shape $B\times256\times T^{\prime}$, alongside Time Cond and Text Emb, the initial features are processed through Patchify and Embed layers. These tokens are then passed through $N$ Transformer Blocks, which incorporate Orbit Transformer Block and Nova Transformer Block designs utilizing Multi-Head Self-Attention and Multi-Head Cross-Attention. After passing through the $N$ layers, latent generation proceeds via a Linear \& Reshape layer, ultimately yielding the Output Latent and Output Velocity in their original $B\times256\times T^{\prime}$ dimensions.}
    \label{fig:appendidx_model}
\end{figure*}

\subsection{Model Parameterization}
\label{appendix_2models}
\textbf{Backbone configuration.} We parameterize the conditional vector field using a Transformer-based flow matching network which shown in \Cref{fig:appendidx_model}. The input latent representations are first linearly projected to a hidden dimension $d_{\text{model}}=512$. The network architecture is composed of $15$ Transformer blocks, each utilizing $8$ attention heads and applying a dropout rate of $0.1$. To ensure training stability, we adopt a pre-normalization layout, applying LayerNorm prior to the sub-modules and employing residual connections afterward.

\textbf{Time embedding parameterization.} The time step $t$ is embedded using sinusoidal Fourier features with a maximum period of $10{,}000$. The resulting concatenated $\sin(\cdot)$ and $\cos(\cdot)$ components are then mapped to the hidden dimension $d_{\text{model}}$ through a two-layer multi-layer perceptron (MLP) with an expansion ratio of $4$ (i.e., $512 \rightarrow 2048 \rightarrow 512$). This resulting time embedding serves as a crucial conditioning signal injected into each block, preserving the latent sequence length throughout the network.

\textbf{Condition builder and injection.} Outputs derived from the text encoder are first projected to $d_{\text{model}}$ utilizing a sequence of Linear and LayerNorm layers. Subsequently, a global text summary is extracted via masked mean pooling. This summary representation is then transformed and additively combined with the time step embedding $t$. To mitigate the risk of early over-conditioning during training, the text contribution is scaled by a learnable gating parameter, initialized at $\alpha_{\text{init}}=0.02$. The resulting combined condition vector is injected into every Transformer block, while the fine-grained, token-level text features are retained to serve as memory for the cross-attention mechanisms.

\textbf{Positional encoding and attention masks.} For the latent sequence, we incorporate learned positional embeddings accommodating a maximum length of $1024$. Conversely, the text memory stream does not require additional positional encoding, as the encoder outputs are inherently position-aware. Both the latent and text streams employ key-padding masks to handle variable sequence lengths. To ensure the integrity of the generation process, a latent validity mask is systematically reapplied following each ODE solver update, guaranteeing that the padded positions remain strictly unchanged.

\textbf{Stability-oriented block design.} To further enhance numerical stability within the network, the Orbit blocks apply QK normalization and facilitate the injection of the combined conditioning signal via adaptive layer normalization (AdaLN). Finally, to optimize inference efficiency, we introduce an option to cache the key-value (KV) tensors associated with the frozen text memory. This efficiently circumvents the need to recompute text-side attention across successive ODE solver steps.

\subsection{Construction of the Harder HumanML3D Test Set}
\label{llm_split}
\textbf{Overview.} To rigorously evaluate the robustness and fine-grained synthesis capabilities of text-to-motion models, we construct the Harder-HumanML3D dataset. This subset is derived from the original HumanML3D test set (4,382 samples) by filtering for motions associated with complex, multi-step, or spatially-precise textual descriptions. Our methodology combines heuristic linguistic rules with large language model (LLM) reasoning to ensure both structural complexity and semantic depth.

\begin{table*}[htbp]
    \centering
    \caption{\footnotesize \textbf{Keyword Lexicon for Multi-Dimensional Difficulty Rules.} We define eight linguistic dimensions ($D_1$--$D_8$) to quantify the complexity of motion prompts. Keywords are matched using heuristic rules to compute the initial difficulty score.}
    \label{tab:keyword_lexicon}
    \renewcommand{\arraystretch}{1.2} 
    \begin{tabularx}{\textwidth}{@{} l >{\raggedright\arraybackslash}X @{}}
        \toprule
        \textbf{Dimension} & \textbf{Matching Keywords / Patterns} \\
        \midrule
        $D_1$ \textbf{Direction}  & left, right, forward, backward, clockwise, counterclockwise, sideways, diagonal, north,,south, east, westward, \dots \\
        $D_2$ \textbf{Body Parts} & arm, leg, hand, foot, head, torso, hip, shoulder, elbow, knee, wrist, ankle, spine, neck, finger, chest, back, \dots \\
        $D_3$ \textbf{Modifiers}  & quickly, slowly, slightly, sharply, gently, rapidly, suddenly, gradually, briefly, continuously, repeatedly, \dots \\
        $D_4$ \textbf{Temporal}   & then, while, before, after, simultaneously, followed by, at the same time, next, finally, meanwhile, \dots \\
        $D_5$ \textbf{Count}      & twice, three times, several times, multiple times, once, again, number of, \dots \\
        $D_6$ \textbf{Verbs}      & kick, punch, bend, twist, spin, crawl, kneel, wave, crouch, squat, hop, leap, lunge, stomp, clap, point, rotate, \dots \\
        $D_7$ \textbf{Constraint} & only, specific, particular, certain, exactly, precisely, must, specifically, \dots \\
        $D_8$ \textbf{Spatial}    & upward, downward, inward, outward, horizontal, vertical, perpendicular, parallel, above, below, overhead, \dots \\
        \bottomrule
    \end{tabularx}
\end{table*}

\textbf{Multi-Dimensional Difficulty Rules.} We define eight linguistic dimensions ($D_1$ to $D_8$) to quantify the complexity of motion prompts, targeting specific challenges in motion synthesis such as directional precision, temporal composition, and spatial relations. The detailed keyword lexicon and the assigned heuristic weight $w_i$ for each dimension are summarized in Table~\ref{tab:keyword_lexicon}. Based on this lexicon, the initial heuristic score is calculated as:
\begin{equation}
    S_{rule} = \sum_{i=1}^{8} w_i \cdot \text{count}(D_i)
\end{equation}

\begin{figure*}[htbp]
    \centering
    \includegraphics[width=\linewidth]{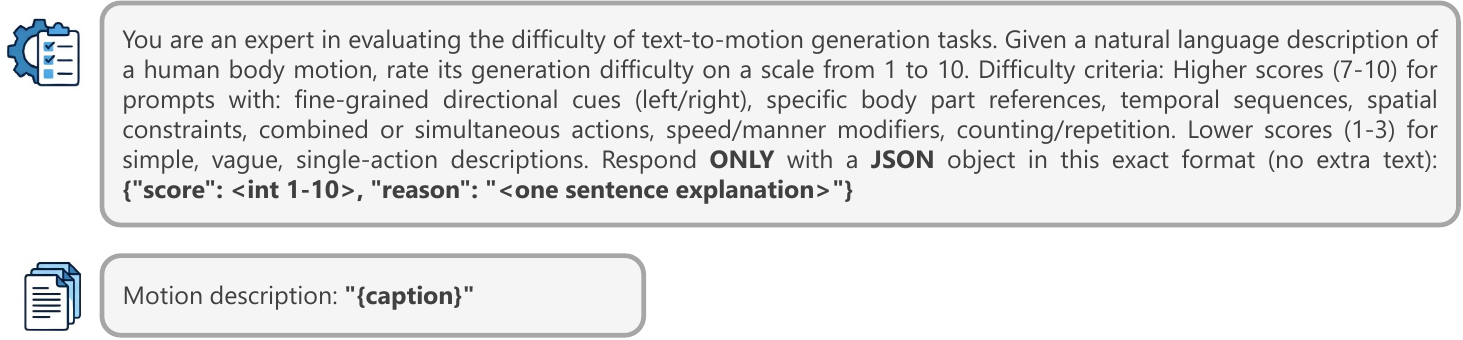}
    \caption{\footnotesize \textbf{Prompt template for Difficulty Oracle.} The system prompt (top) defines linguistic evaluation criteria and enforces JSON output, while the user prompt (bottom) supplies the target caption to enable structured semantic scoring.}
    \label{fig:llm_prompt}
\end{figure*}

\textbf{LLM-Assisted Semantic Scoring.} Recognizing that word counts alone cannot fully capture logical complexity, we employ \texttt{DeepSeek-R1} to oracle the difficulty of the prompts. For candidates with $S_{rule} \geq 4$, the LLM evaluates the prompt's difficulty on a scale of 1–10 and provides a natural language justification. The complete prompt template designed for this evaluation is illustrated in Figure~\ref{fig:llm_prompt}. This crucial step filters out prompts that contain difficulty keywords but remain semantically simple (e.g., ``a person walks left'').

\textbf{Final Selection and Dataset Statistics.} The final difficulty score is a fusion of the heuristic and LLM scores:
\begin{equation}
    S_{final} = S_{rule} + \alpha \cdot S_{llm}
\end{equation}
where we set $\alpha = 1.0$. Following the protocol of Fg-T2M \cite{Fg_T2M}, we select the top 2,582 unique motion sequences with the highest $S_{final}$. 

Compared to the full test set, the Harder-HumanML3D subset exhibits a significantly higher density of directional terms and complex temporal structures, posing a substantial challenge for current state-of-the-art models.


\end{document}